\documentclass{article}

\usepackage{microtype}
\usepackage{graphicx}
\usepackage{subfigure}
\usepackage{multirow}
\usepackage{booktabs} 
\usepackage[scaled=0.85]{beramono}
\usepackage[most]{tcolorbox}
\usepackage{hyperref}

 \usepackage[preprint]{neurips_2026}

\usepackage[utf8]{inputenc} 
\usepackage[T1]{fontenc}    
\usepackage{hyperref}       
\usepackage{url}            
\usepackage{booktabs}       
\usepackage{amsfonts}       
\usepackage{nicefrac}       
\usepackage{microtype}      
\usepackage{xcolor}         
\usepackage{amsmath}
\usepackage{amssymb}
\usepackage{mathtools}
\usepackage{amsthm}
\usepackage{bbm}
\usepackage{stmaryrd}
\usepackage{mathrsfs}
\usepackage{dsfont}
\usepackage{bm}
\usepackage{mdframed}
\usepackage{thm-restate}
\usepackage{enumitem}
\usepackage{wrapfig}

\usepackage[capitalize,noabbrev]{cleveref}

\theoremstyle{plain}
\newtheorem{theorem}{Theorem}[section]
\newtheorem{proposition}[theorem]{Proposition}

\newtheorem{corollary}[theorem]{Corollary}
\theoremstyle{definition}
\newtheorem{definition}[theorem]{Definition}
\newtheorem{assumption}[theorem]{Assumption}
\theoremstyle{remark}
\newtheorem{remark}[theorem]{Remark}

\newcommand{\eqdef}{\stackrel{\text{def}}{=}}

\newcommand{\bistar}{\masksample_{i_{\ast}}}

\newcommand{\Prob}[1]{\mathbb{P}\bigl[#1\bigr]}
\newcommand{\Ind}[1]{\mathds{1}\bigl\{#1\bigr\}}
\newcommand{\istar}{i_{\ast}}
\newcommand{\algo}{\mathcal{A}}
\newcommand{\dataset}{\mathbb{R}^{n \times d}}

\newcommand{\Hstar}{H_{\ast}}
\newcommand{\qstar}{p_{\ast}}
\newcommand{\data}{\underline{\mathbf{z}}}
\newcommand{\res}{\underline{\mathbf{r}}}
\newcommand{\datadomain}{\mathcal{Z}}
\newcommand{\missdatadomain}{\mathcal{Z}_{\text{miss}}}
\newcommand{\sample}{\mathbf{z}}
\newcommand{\datavar}{\bm{Z}}
\newcommand{\samplevar}{Z}
\newcommand{\masksample}{\mathbf{m}}
\newcommand{\mask}{\underline{\mathbf{m}}}
\newcommand{\Miss}{\mathcal{A}_{\datamiss}}
\newcommand{\datamiss}{\mathcal{D}}
\newcommand{\featuremiss}{\mathcal{F}}

\newcommand{\MargZ}{\Prob{\datamiss(\data) = \mask}\Prob{\algo(\tilde{\data}(\mask)) \in S}}
\newcommand{\MargZz}{\Prob{\datamiss(\data') = \mask}\Prob{\algo(\tilde{\data}'(\mask)) \in S}}
\newcommand{\Pmx}[1]{\mathbb{P}\bigl[{\algo(\tilde{\data}(\mask))} \in #1\bigr]}
\newcommand{\Pmxx}[1]{\mathbb{P}\bigl[{\algo(\tilde{\data}'(\mask))} \in #1\bigr]}
\newcommand{\Divmiss}[1]{D_{#1}(\algo_{\datamiss}(\data) || \algo_{\datamiss}(\data'))}
\newcommand{\Pfeat}[1]{\mathbb{P}\bigl[\datamiss(#1) = \mask\bigr]}
\newcommand{\Bin}{\mask \in \{0, 1\}^{n \times d}}
\newcommand{\Biin}{\mask' \in \{0, 1\}^{n \times d}}
\newcommand{\query}{f_{I}}
\newcommand{\querysens}[1]{\Delta_{\query}^{(#1)}}
\newcommand{\querysensmiss}[1]{\tilde{\Delta}_{\query}^{(#1)}}
\newcommand{\lap}{\algo^{\mathcal{L}}}
\newcommand{\lapmiss}{\Miss^{\mathcal{L}}}
\newcommand{\gauss}{\algo^{\mathcal{G}}}
\newcommand{\guassmiss}{\Miss^{\mathcal{G}}}
\newcommand{\datamask}{\tilde{\data}}
\newcommand{\samplemask}{\tilde{\sample}}
\newcommand{\Divmask}[1]{D_{#1}(\algo(\datamask) || \algo(\datamask'))}

\definecolor{MDnavy}{RGB}{0,87,168}
\definecolor{DPcrimson}{RGB}{221,0,0}
\usepackage[textsize=tiny]{todonotes}

\usepackage{xr}
\usepackage{xcolor}

\newboolean{showcomments}
\setboolean{showcomments}{true}
\ifthenelse{\boolean{showcomments}}
{ \newcommand{\mynote}[3]{
		\fbox{\bfseries\sffamily\scriptsize#1}
		{\small$\blacktriangleright$\textsf{\emph{\color{#3}{#2}}}$\blacktriangleleft$}}
	\newcommand{\zzz}[1]{{\setlength{\fboxsep}{2pt}\fcolorbox{black}{yellow}{\textsf{\emph{#1}}}}\xspace}}
{ \newcommand{\mynote}[3]{}
	\newcommand{\zzz}[1]{}}

\title{On the Inherent Privacy Amplification of Missing Data}

\author{%
  Simon Roburin \quad Rafael Pinot \quad Erwan Scornet \\ \\
  Sorbonne Université, Université Paris Cité, CNRS \\
  Laboratoire de Probabilités, Statistique et Modélisation (LPSM) \\
  F-75005 Paris, France 
}

\begin{document}

\maketitle

\begin{abstract}
Privacy preservation is critical in many high-stakes domains such as medicine and finance, where sensitive data must be analyzed without compromising individual confidentiality. At the same time, these applications often involve datasets with inherent missing values due to non-response or data corruption for example. Missing data is traditionally analyzed through its impact on statistical efficiency and model performance. In fact, it reduces the information available to analysts and can degrade the final utility of the model. In this work, we take an alternative approach and study missing data through the lens of privacy preservation. 
Intuitively, when features are missing, less information is revealed about individuals, suggesting that data missingness could inherently enhance privacy. We formalize this intuition within a novel framework that integrates missing data into differential privacy. In essence, our approach accounts for the de facto missingness in the data to refine existing privacy guarantees without modifying the underlying mechanism. Using this framework, we show for the first time that missing data can induce inherent privacy amplification for differentially private algorithms, highlighting a previously overlooked interaction between missing data and formal privacy guarantees.
\end{abstract}

\section{Introduction}
Privacy preservation is a fundamental requirement in many high-stakes applications, where the misuse or unintended disclosure of sensitive information can have severe consequences. In fact, domains such as medicine \citep{el2011systematic}, genomics \citep{homer2008resolving}, finance \citep{wang2018privacy}, and social sciences \citep{abowd2023confidentiality} routinely involve the collection and analysis of highly sensitive personal data, making strong privacy guarantees an indispensable component of modern data analysis pipelines. Beyond ensuring that private information is not directly leaked or accidentally disclosed, a fundamental issue of machine learning approaches is to ensure that information cannot be recovered or inferred from the sole outputs of data analysis procedures. Differential privacy (DP), introduced by \cite{dwork2006differential}, has become the gold standard across multiple research communities by providing a rigorous mathematical definition of privacy. Informally, DP bounds the shift in the output distribution of any data analysis procedure induced by altering a single (differing) record in the data. Thereby, DP guarantees statistically similar outputs regardless of whether an individual’s data is included in or excluded from the dataset. This similarity is quantified by two privacy parameters $\epsilon$ and $\delta$; the smaller these parameters, the stronger the privacy protection.

At the same time, privacy-critical applications are frequently characterized by the presence of missing data. Data may be missing due to non-response, data corruption, selective disclosure, or deliberate data anonymization practices motivated by ethical, legal, or operational constraints \citep{little2002statistical}. As a result, data missingness is not an exceptional phenomenon but rather an inherent feature of many real-world datasets. To characterize this phenomenon and quantify its implications on data analysis, the classical literature on missing values defines missing data scenarios based on the relationship between missingness and observed values \citep{rubin1976inference}. If the missingness is independent of both observed and unobserved data, the setting is called Missing Completely At Random (MCAR); if it depends only on the observed values, it is referred to as Missing At Random (MAR); otherwise, it is classified as Missing Not At Random (MNAR). In all of these scenarios, missing data is most often viewed as a limitation because it reduces the amount of information available to the analyst and, consequently, can degrade the accuracy of the models one aims to design. Accordingly, substantial effort has therefore been devoted to characterizing and mitigating its negative impact through imputation \citep{little2002statistical, mattei2019miwae}, data augmentation \citep{garcia2010missing}, or robust learning techniques \citep{che2018recurrent,you2020handling}; see also \citet{sun2023deep} and \citet{leimputation} for a recent survey.

\subsection{Contributions}
In this paper, we adopt a somewhat unconventional perspective and argue that the inherent missing data that we have to face in many real-world applications can also be considered from a privacy preservation standpoint. Intuitively, when some features are missing for certain individuals, the amount of information available about them is reduced, which can inherently protect their privacy. While intuitive, this idea is difficult to formalize in a way that yields rigorous privacy guarantees. We address this challenge by studying privacy protection through the notion of differential privacy. Within this setting, we show that although missing data does not provide privacy on its own, it can be interpreted and analyzed as a form of privacy amplification. Our contributions are as follows. 

\textbf{1. Framing missing data within the differential privacy terminology.} In Section~\ref{sec:DP_miss}, we first adapt the standard definition of differential privacy to account for missing data in the inputs. This allows us to formalize randomized algorithms under missing data as the composition of a missing data scenario and a downstream randomized algorithm that satisfies differential privacy. This unified probabilistic framework enables a rigorous and systematic analysis of the interaction between missingness and privacy within a single model. To the best of our knowledge, this is the first general formalization that explicitly treats missing data as an integral component of the privacy pipeline. This perspective enables a fine-grained characterization of how inherent data missingness may influence the resulting privacy guarantees of the pipeline, and lays the foundation for our privacy amplification results. 

\textbf{2. Proving the privacy amplification effect of missing data in general.} In Section~\ref{sec:amplif}, we quantify the impact of missing data on the privacy guarantees of any differentially private algorithm. We show that in MCAR, MAR, and MNAR settings, incomplete data can induce privacy amplification. This means that, explicitly accounting for missingness (rather than ignoring it, as is standard in prior works) yields stronger privacy guarantees, with smaller privacy parameters $\epsilon$ and $\delta$ than those of the original algorithm. The amplification factor generally depends on the structure of the missing data process, but in MAR and MCAR settings, this factor is directly governed by the probability that missingness masks the features on which two neighboring datasets differ. Finally, we identify regimes where this amplification is limited or negligible, motivating a more refined analysis for restricted classes of differentially private algorithms and query families.

\textbf{3. Amplification for feature-wise Lipschitz queries.} In Section~\ref{sec:fwl}, we focus on a class of queries that we call \emph{feature-wise Lipschitz} (FWL), which encompasses many widely studied queries in the DP literature, including histograms, linear queries, coordinate-wise clipped means, and covariance queries. For this class, we establish stronger privacy amplification results for the standard Laplace and Gaussian mechanisms. Specifically, we show in MCAR and MAR scenarios that the privacy amplification factor for FWL queries depends directly on the fraction of missing observations per sample. This means that, even when the general privacy amplification results from Section~\ref{sec:amplif} are limited, privacy amplification occurs for standard differentially private mechanisms.

Overall, we provide a principled understanding of how data missingness interacts with privacy, and highlight the opportunities it presents. Our analysis isolates the amplification arising solely from natural missingness in the data, without introducing additional randomness or intervention from the data curator. As a result, our guarantees strictly refine existing differential privacy analyses by incorporating previously overlooked structure, while leaving the underlying mechanisms unchanged.

\subsection{Closely related work}
The interaction between DP and missing data has primarily been studied from the perspective of mitigating the negative impact of missingness on either privacy guarantees or utility. Early work focuses on the interplay between DP and imputation procedures. \citet{clifton2022differentially} show that even simple imputation strategies can significantly compromise privacy, as modifying a single record may influence multiple imputed values. More generally, \citet{das2022imputation} establish that, for arbitrary imputation procedures, applying a differentially private algorithm to imputed data leads to privacy degradation that scales with the number of imputed entries affected by a single individual. More recently, and most closely related to our work, \citet{mohapatra2023differentially} study missing data in the context of synthetic data generation. Under an MCAR assumption, they observe privacy gains that can be interpreted as a consequence of random subsampling, effectively modeling missingness as a Bernoulli thinning process over records. However, this perspective is limited to specific mechanisms and relies on strong assumptions about the missingness process. In contrast, we adopt a fundamentally different viewpoint. We develop a general probabilistic framework that explicitly incorporates missingness into the DP analysis. Our results apply not only under MCAR, but also under MAR and MNAR settings, and reveal regimes in which privacy amplification depends directly on the probability that missingness obscures the features distinguishing neighboring datasets. 

Our work is also related to the classical literature on privacy amplification in DP. A well-studied instance arises from random subsampling, where a DP mechanism is applied to a randomly selected subset of the data \citep{li2012sampling, abadi2016deep, balle2018privacy}. Another prominent line of work studies amplification by shuffling \citep{erlingsson2019amplification,cheu2019distributed}, where privacy gains result from random permutation of record indices prior to applying a DP mechanism. These results are typically tied to record-level algorithmic primitives and do not account for randomness induced by the data generation or observation process itself. In contrast, we identify missing data as a distinct and inherent source of privacy amplification. Unlike subsampling or shuffling, this effect arises at the feature level through stochastic masking of attributes, and is governed by the probability distribution of the missingness mask. From a technical standpoint, although we build on standard tools from the differential privacy literature such as $\alpha$-divergence, handling missing data requires new methodological ingredients. In particular, we introduce a unified mixture framework over masking patterns, along with mask-specific couplings that preserve both the missingness process and the neighboring definition for incomplete datasets (see Appendix~\ref{app:proof}). We also develop auxiliary distribution arguments to address the MNAR setting, where data-dependent missingness induces distribution mismatches (see also Appendix~\ref{app:proof} for details). For the MAR and MCAR regimes, we complement this analysis with a structural lemma (see \Cref{lemma:pstar}) showing that the probability of revealing the differing record can be characterized by a dataset-independent quantity $p^\star$. This yields a simpler and more stable decomposition, enabling sharper privacy bounds without the need for auxiliary alignment arguments.  These techniques are novel and may be useful for further advancing the study of interactions between missing data and privacy preservation. Furthermore, from a conceptual standpoint, our analysis captures amplification induced solely by natural missingness in the data, whereas classical amplification techniques introduce additional algorithmic layers specifically designed to enhance privacy.\vspace{-8pt}
\vspace{-5pt}
\section{\vspace{-6pt}Differential privacy with missing data\vspace{-2pt}}
\label{sec:DP_miss}
\vspace{-4pt}
In this section, we introduce the necessary formalism to present our main technical contributions on the amplification effect of incomplete data on differential privacy. Following the classical literature on missing values, we first present how data missingness can be modeled through a process defined as a conditional distribution of a binary mask that depends on the complete data~\citep{rubin1976inference}. Then, to build an object that accounts for missing data and privacy at once, we introduce a randomized algorithm that corresponds to the composition of this process and a DP mechanism.

\textbf{Preliminary notations.} 
Let $\datadomain \subset \mathbb{R}^d$ be the data domain. For any integer $n \geq 1$, we use the notation $[n]$ to denote the set $\{1,\ldots,n\}$. We also denote by $\mathbf{1}_{\mathbb{R}^d} \in \mathbb{R}^d$ the $d$-dimensional vector whose entries are all equal to $1$. By convention, for any vector $\mathbf{v} \in \mathbb{R}^d$, we denote by $\mathbf{v}^{(j)}$ its $j$-th coordinate for all $j \in [d]$. Accordingly, we also let, for any $J \subset \{1, \hdots, d\}$, $\mathbf{v}^{(J)}$ be the subvector composed of the components of $\mathbf{v}$ indexed by $J$. Furthermore, we denote by $\textrm{Card}(\cdot)$ the function that computes the cardinality of a set. 
Lastly, let $\datavar = (\samplevar_1,\ldots,\samplevar_n)$ be a random vector taking values in $\datadomain^n$, where for all $i \in [n]$, $\samplevar_i$ is a random variable taking values in $\datadomain$. A dataset $\data = (\sample_1, \ldots, \sample_{n}) \in \datadomain^{n}$ is a realization of $\datavar$, and for all $i \in [n]$, $\sample_i$ is a realization of $\samplevar_i$.

\subsection{Incomplete datasets}

To define data incompleteness, we focus on missing data processes that act randomly on the features of the dataset and introduce the missing data domain $\missdatadomain \subset (\mathbb{R} \cup \{\texttt{NA}\})^{d}$ where $\texttt{NA}$ denotes a symbol representing missing values. We extend the usual arithmetic on $\mathbb{R}$ with the conventions 
\begin{align*}
  \texttt{NA}\cdot 1 = \texttt{NA}, 
\quad \texttt{NA} + 0 = \texttt{NA} \quad \text{and}\quad
\texttt{NA} \cdot 0 = 0.
\end{align*}
Following the notation we adopt for samples and datasets, we denote by $\mathbf{\textbf{\texttt{NA}}} = \{\texttt{NA}\}^d$ the vector whose entries are all equal to \texttt{NA}, and by $\underline{\textbf{\texttt{NA}}} = \{\texttt{NA}\}^{n\times d}$ the dataset whose entries are all equal to \texttt{NA}. As we just mentioned, in presence of missing values, we do not observe a complete dataset $\data \in \datadomain^{n}$, but rather an incomplete dataset $\tilde{\data} \in \missdatadomain^{n}$. To formalize the construction of $\tilde{\data}$, we first define the masking operation at the sample level and then extend it to the dataset level. Let $\sample \in \datadomain$ be an arbitrary sample. An incomplete observation of $\sample$ can be defined through a missing sample mask $\masksample \in \{0,1\}^d$, where for all $j \in [d],$ we have $\masksample^{(j)} = 1$ if the $j$-th feature of $\sample$ is not observed and $\masksample^{(j)} = 0$ otherwise.
Accordingly, the incomplete sample in $\missdatadomain$ corresponding to $\sample$ and $\masksample$ writes \vspace{-7pt}

\begin{align*}
  \tilde{\sample}(\masksample) = \sample \odot (1 - \masksample) + \textbf{\texttt{NA}} \odot \masksample,
\end{align*}
where $\odot$ is the element-wise multiplication. Let us denote by $\operatorname{obs}(\masksample) = \{j \in [d] \mid \mathbf{m}^{(j)} = 0 \}$, the set of observed features
given a mask $\masksample \in \{0,1\}^d$. We can also define the observed sample for $\sample$ as the concatenation of the observed features of $\sample$, i.e., $\sample^{\operatorname{obs}(\masksample)} = (\sample^{(j)})_{j \in \operatorname{obs}(\masksample)}$. Finally, we can extend the previous definition to a dataset $\data = (\sample_1, \dots, \sample_n) \in \datadomain^n$. Given a missing data mask $\mask = (\masksample_1, \dots, \masksample_n)^\top \hspace{-2pt} \in \{0,1\}^{n\times d}$, we define the incomplete dataset $\tilde{\data}(\mask) \in \missdatadomain^n$ as \vspace{-7pt}

\begin{align*}
  \tilde{\data}(\mask) = \data \odot (1-\mask) + \underline{\textbf{\texttt{NA}}}\odot \mask.
\end{align*} 

When explicit dependency on $\sample$ and $\masksample$ are not needed, we denote by $\samplemask$ any incomplete sample from $\missdatadomain$. Similarly, $\datamask$ will denote any incomplete dataset belonging to $\missdatadomain^n$.

\subsection{Missing data process}\label{sec:datamiss}

As first observed by~\citet{rubin1976inference,little2002statistical}, to reason about data missingness, one needs to put a model on how missing data occur. In the context of statistical learning, missingness is first modeled at the sample level by the random feature masking process $\featuremiss$, which generates a mask for each sample. Then all the masks are drawn independently of each other and are then combined to define the missing data process, that we denote $\datamiss$ or equivalently $\featuremiss^{\otimes n}$ (by overloading the notation $\featuremiss$). This construction is summarized in the diagram below. 

\begin{wrapfigure}[10]{r}{.45\textwidth}
\vspace{-10pt}
\begin{align}
\textstyle
\left.
\begin{array}{ccc}
\sample_1 & \xrightarrow[\text{indep. draws}]{\;\featuremiss\;} & \masksample_1\\ 
\vdots & & \vdots \\ 
\sample_n & \xrightarrow[\text{indep. draws}]{\;\featuremiss\;} & \masksample_n 
\end{array} \nonumber 
\right\} \quad \equiv \quad \data \xrightarrow[\text{draws}]{\;\datamiss\;} \mask 
\end{align} \vspace{3pt}
\begin{equation}
\small
    \Prob{\featuremiss(\sample)=\masksample} \eqdef \Prob{\featuremiss(Z)=\masksample \mid Z=\sample}. \label{eq:feature_miss} 
\end{equation}
\end{wrapfigure}

The missing feature process $\featuremiss$ is a randomized process that takes as input a sample $\sample \in \datadomain$ and outputs a random variable taking values in $\{0,1\}^d$ interpreted as a random mask. To characterize missing data processes, we must adopt a probabilistic viewpoint on data where any sample $\sample$ is considered as a realization of a $\datadomain$-valued random variable $Z$. We thus denote the conditional distribution of the output of the random mask given $Z$ as per~\eqref{eq:feature_miss}. Based on this, the missing data process $\datamiss$ takes as input a dataset $\data$ in $\datadomain^n$ and outputs a random vector taking values in $\{0,1\}^{n \times d}$. 
For each $i \in [n]$, the masks $\masksample_i$ are drawn independently according to the law of the random variable $\featuremiss(\sample_i)$. Formally, as observed in~\citet{little2002statistical}, for all $\mask = (\masksample_1, \dots,\masksample_n)^\top$ and for all $\data = (\sample_1, \dots, \sample_n) \in \datadomain^n$, since the masks are drawn independently, using the above notations, we have 
\begin{align}
  \Prob{\datamiss(\data) = \mask} 
  &= \prod_{i=1}^n \Prob{\featuremiss(\sample_i) = \masksample_i}.\label{def:missingmech}
\end{align}
Depending on how $\featuremiss$ interacts with its input $\sample_i$, we will need to distinguish in this work, three main scenarios, namely \emph{Missing Completely At Random}, \emph{Missing At Random}, and \emph{Missing Not At Random}.
\begin{itemize}[leftmargin=*]
  \item \textbf{Missing Completely At Random (MCAR). } $\mathcal{F}$ is said to be MCAR if the missingness mask is independent of the data. Specifically, this means that for all sample $\sample, \sample' \in \datadomain$ and all mask $\masksample \in \{0,1\}^d$, we have $\Prob{\featuremiss(\sample) = \masksample} = \Prob{\featuremiss(\sample') = \masksample}.$  \\
  
  \item \textbf{Missing At Random (MAR).} $\mathcal{F}$ is said to be MAR if the missingness mask depends on the observed values. Specifically, for all sample $\sample, \sample' \in \datadomain$
  and all mask $\masksample \in \{0,1\}^d$ such that $
\sample^{\operatorname{obs}(\masksample)} = \sample^{\prime\operatorname{obs}(\masksample)}$, 
  we have $ \Prob{\featuremiss(\sample) = \masksample} = \Prob{\featuremiss(\sample') = \masksample}.$ \\
  
  \item \textbf{Missing Not At Random (MNAR).} 
  $\mathcal{F}$ is said to be MNAR if $\mathcal{F}$ is not MAR.
\end{itemize}
Thus, we see that MCAR is included in MAR, whereas MNAR is the complement of MAR. 
Using the above, a missing data process $\datamiss = \featuremiss^{\otimes n}$ is said to be MCAR (respectively, MAR or MNAR) if the missing feature process $\featuremiss$ is MCAR (respectively, MAR or MNAR).

\subsection{Adapting the definition of differential privacy }\label{sec:algomiss}

To account for both the missing data process $\datamiss$ and the randomized algorithm $\algo$ usually used for differential privacy, we now define the randomized algorithm $\Miss$. As summarized in the diagram below, $\Miss$ is the method that characterizes the application of a (possibly differentially private) algorithm $\algo$ to a masked sample, assuming the mask has been produced by a missing data process $\datamiss$. 
\begin{align}
\label{def:A_D}
\textstyle
\underbrace{ \textcolor{MDnavy}{\data 
\;\xrightarrow[\text{draws}]{\datamiss\;}\; \mask
\;\xrightarrow[\text{masks}]{\;\data\;}\; \datamask(\mask)} \textcolor{DPcrimson}{
\;\xrightarrow[\text{releases}]{\;\algo\;}\; \algo\left(\datamask(\mask)\right)}}_{\textstyle\textcolor{DPcrimson}{\mathcal{A}_{\textcolor{MDnavy}{\mathcal{D}}}}(\data)}.
\end{align} 
More precisely, $\Miss$ is defined as follows. First, the complete dataset $\data$ in $\datadomain^n$ is transformed into an incomplete dataset $\tilde{\data} \in \missdatadomain^n$ by sampling a mask $\mask \sim \datamiss(\data)$ from the missing data process $\datamiss$. This part (\textcolor{MDnavy}{in blue in the diagram}) is what characterizes the natural randomness arising from the fact that we deal with missing data in the analysis. Second, the masked dataset $\tilde{\data}(\mask)$ is fed to the randomized algorithm $\algo$ to obtain $\algo(\tilde{\data}(\mask))$. This part (\textcolor{DPcrimson}{in red in the diagram}) characterizes the differentially private data processing made by an analyst when facing an incomplete database. To formalize this, we need to define the notion of differential privacy for mechanisms applying in the missing data domain $\missdatadomain$. To do so, we equip the dataset domain $\missdatadomain^n$ with the ``substitute one record" binary symmetric relation denoted $\simeq$, which defines the notion of neighboring elements within $\missdatadomain$. This definition is standard in the differentially private community. Nevertheless, for completeness, we provide a formal definition in Appendix~\ref{app:dist}. Because any pair of neighboring datasets $\datamask$ and $\datamask'$ can be aligned by an appropriate permutation, we assume without loss of generality that $\datamask \simeq \datamask'$ if and only if they differ in at most one coordinate in $[n]$. Accordingly, we define differential privacy below. 
\smallskip

\begin{definition}
  Let $\epsilon > 0$ and $\delta \in [0,1]$. Let $\algo$ be a randomized algorithm and let $\mathcal{E}$ be the set of all measurable output sets for $\algo$. $\algo$ is $(\epsilon, \delta)$-differentially private, or $(\epsilon, \delta)$-DP, if for any two neighboring datasets $\datamask \simeq\datamask'$ in $\missdatadomain^n$ and for any $S \in \mathcal{E}$, we have
  \begin{align}
  &\Prob{\algo(\datamask) \in S} 
  \;\leq\; \nonumber e^{\epsilon} \, \Prob{\algo(\datamask') \in S} + \delta. 
\end{align}
\label{def:DP}
\end{definition} \vspace{-15pt}

In the above, the privacy budget $(\epsilon, \delta)$ measures the amount of privacy the algorithm holds. The parameter $\epsilon$ controls the strength of the privacy guarantee, while $\delta$ is often described as the probability of a rare privacy violation. The smaller both of these parameters are, the stronger privacy the algorithm guarantees. In what follows, to manipulate the privacy budget of an algorithm, we also express differential privacy as an $\alpha$-divergence as first introduced by \citet{barthe2013beyond}. 



\begin{theorem}[\citet{barthe2013beyond}] Let $\epsilon >0 $ and $\delta \in [0,1]$. For any randomized algorithm $\algo$ defined over $\missdatadomain^n$, and any $\datamask, \datamask' \in \missdatadomain^n$, we denote 
 \begin{align*}
    & \Divmask{e^\epsilon} =  \sup_{S \in \mathcal{E}} \Prob{\algo(\datamask) \in S} - e^{\epsilon}\Prob{\algo(\datamask') \in S} .
  \end{align*}
Using this notation, an algorithm $\algo$ is $(\epsilon, \delta)$-DP if and only if $\delta_{\algo}(\epsilon) \leq \delta$, where 
  \begin{equation*}
    \delta_{\algo}(\epsilon) = \sup_{\datamask \simeq \datamask'}
    \Divmask{e^\epsilon}.
  \end{equation*}
  \label{thm:dp_barthe}
\end{theorem} \vspace{-7pt}

\begin{remark}
    Note that, by construction, any dataset in $\datadomain^n$ can be viewed as an element of $\missdatadomain^n$ with no missing entries. Hence, although \Cref{def:DP} and \Cref{thm:dp_barthe} are stated on $\missdatadomain^n$, they can directly be applied to datasets in $\datadomain^n$. In the remainder, we will use the concept of DP in both settings.
\end{remark}

\section{Privacy amplification of missing data on generic DP algorithms}\label{sec:amplif}



In this section, we present our first result quantifying how accounting for the missing data process affects the privacy guarantees of any differentially private algorithm. Specifically, we show that in both the MNAR and MAR settings (and therefore also in MCAR), $\datamiss$ can be viewed as inherently amplifying the privacy of a $(\epsilon,\delta)$-DP algorithm $\algo$. That is, the induced mechanism $\Miss$ satisfies $(\epsilon',\delta')$-DP with $\epsilon' \leq \epsilon$ and $\delta' \leq \delta$. To capture this effect, it is essential to examine the distribution of the masks in cases where the differing record between two neighboring datasets is at least partially observed. 
To this end, for any pair of neighboring datasets $\data \simeq \data'$ in $\datadomain^n$, we introduce the set $\Hstar(\data,\data')$ corresponding to the event that the differing record is at least partially observed, i.e., \vspace{-3pt}

\begin{equation}
  \Hstar(\data,\data') = \big\{ \mask \in \{0,1\}^{n\times d}\mid\masksample_{i_{\ast}} \neq \mathbf{1}_{\mathbb{R}^d} \text{ for all $i_{\ast}$ such that }  \data_{i_{\ast}} \neq \data'_{i_{\ast}}\big\}. \label{eq:defhc}
\end{equation}

Then we consider the probability  $p_{\data\mid\data'} = \Prob{\datamiss(\data) \in \Hstar(\data, \data')}$  of a mask falling within this set $\Hstar(\data,\data')$, when the mask is generated according to the missing data process on $\data$. Similarly, we consider  $p_{\data'\mid\data} = \Prob{\datamiss(\data') \in \Hstar(\data, \data')}$. With these notations at hand, we can present our first result in Theorem~\ref{thm:MNAR} below. This theorem bounds the divergence between $\Miss(\data)$ and $\Miss(\data')$ for all neighboring $\data$ and $\data'$ and regardless of the scenario we consider. Then according to whether the missing data process is MNAR, MAR or MCAR, this result translates to different privacy guarantees. 

\begin{restatable}{theorem}{MNARThm}\label{thm:MNAR}
    Let $\epsilon>0$ and $\delta\in[0,1]$. Let $\algo$ be an $(\epsilon,\delta)$-DP randomized algorithm on $\missdatadomain^n$. Let $\datamiss=\featuremiss^{\otimes n}$ be any missing data process (MNAR, MAR or MCAR). Let $\Miss$ be the induced mechanism defined as per~\eqref{def:A_D} with $\algo$ and $\datamiss$.
For every neighboring pair $\data\simeq\data' \in \datadomain^n$ and $\lambda \in [0,1]$, we have
\begin{align*}
\Divmiss{e^{\epsilon(\lambda, \data \vert \data')}} \leq \lambda \delta + ( 1 - \lambda) \delta_{\data\mid \data'},
\end{align*}
with $\epsilon(\lambda, \data \vert \data') = \ln(\lambda e^{\epsilon} + (1-\lambda)e^{\epsilon'(\data \mid \data')})$ where $\epsilon'(\data \vert \data') = \log(1 + p_{\data \mid \data'}(e^{\epsilon} -1))$ and with $\delta_{\data\mid \data'} = p_{\data \mid \data'}\delta 
+
e^{\epsilon'(\data\mid \data')}
|p_{\data \mid \data'}-p_{\data' \mid \data}|$.
\end{restatable}

\subsection{Interpretation  in terms of privacy amplification for MNAR.}
As previously mentioned, \Cref{thm:MNAR} allows us to derive privacy amplification results depending on the missing data scenario. In MNAR settings, it suffices to take the supremum on all neighboring couples $\data$, $\data'$ to show that for all $\lambda \in [0,1]$, $\Miss$ satisfies $(\epsilon_{\text{\tiny MNAR}}(\lambda), \delta_{\text{\tiny MNAR}}(\lambda))$-DP with \vspace{-10pt}

\begin{align}
    \epsilon_{\text{\tiny MNAR}(\lambda)} = \sup_{\data\simeq \data'} \ln \left (\lambda e^\epsilon + (1-\lambda) e^{\epsilon'(\data \vert \data')} \right) \text{  and  } \delta_{\text{\tiny MNAR}}(\lambda) = \lambda \delta + (1-\lambda) \sup_{\data \simeq \data'} \delta_{\data\vert\data'}. \label{eq:ampMNAR}
\end{align}

This result characterizes a trade-off curve rather than a single privacy guarantee. In particular, we always have $\epsilon_{\text{\tiny MNAR}(\lambda)} \leq \epsilon$. When $\delta_{\data\vert\data'}$ is strictly smaller than $\delta$, setting $\lambda = 0$ is optimal and yields strict privacy amplification. In contrast, when $\sup \delta_{\data\vert\data'} \geq \delta$ in~\eqref{eq:ampMNAR}, one may trade a slight increase in the $\delta_{\text{\tiny MNAR}}$ term for an improvement in $\epsilon_{\text{\tiny MNAR}}$ by choosing $\lambda \in (0,1)$. Finally, if the missing data process degrades $\sup \delta_{\data\vert\data'}$ too much, setting $\lambda = 1$ recovers the original guarantee, with neither amplification nor degradation. While this provides a general data-dependent characterization, the resulting bound may be difficult to quantify for arbitrary MNAR mechanisms, as it requires uniform control of the mask probabilities $p_{\data \mid \data'}$, $p_{\data' \mid \data}$, and their discrepancy over neighboring datasets. This dependence is nevertheless natural, as the MNAR setting is highly general, and any meaningful privacy amplification result must reflect the specifics of the underlying missing data process.

\subsection{Interpretation in terms of privacy amplification for MAR (and MCAR).}  

However, in MAR and MCAR settings,  the probabilities $p_{\data \mid \data'}$ and $p_{\data' \mid \data}$ simplify considerably. In fact,
it is possible to show that (see Lemma~\ref{lemma:pstar2} in Appendix~\ref{app:lemmapstar}) there exists a constant $\qstar\in[0,1]$ such that for all neighboring
datasets $\data\simeq\data'$ in $\datadomain^n$, we have \vspace{-10pt}

\begin{align}
\qstar = \Prob{\datamiss(\data)\in \Hstar(\data,\data')}
=
\Prob{\datamiss(\data')\in \Hstar(\data,\data')} \label{eq:pstardef}.
\end{align}


This observation allows us to present a data-independent characterization of the impact on privacy of MAR (or MCAR) missing data processes. Specifically, under the conditions of Theorem~\ref{thm:MNAR}, we have that $\Miss$ is $(\epsilon', \delta')$-DP with
\begin{equation}
  \epsilon' = \ln\!\big(1+p_{\ast}(e^{\epsilon}-1)\big) \text{ and } \delta' = p_{\ast}\delta. \label{eq:ampMAR}
\end{equation}

The precise characterization behind this statement can be found in Corollary~\ref{thm:amplif} in Appendix~\ref{app:lemmapstar}. This shows that privacy amplification occurs for any MAR (or MCAR) process and any $(\epsilon, \delta)$-DP algorithm, as long as the probability of not observing the differing record $i_{\ast}$ is non-zero, i.e., $\qstar< 1$. In this case, the resulting randomized algorithm $\Miss$ satisfies $(\varepsilon', \delta')$-DP with $\epsilon' < \epsilon$ and $\delta' < \delta$. 
Although one may legitimately wonder whether $\qstar < 1$ in practice, we exhibit real-world configurations in Appendix~\ref{app:rwdata} where this condition holds. Importantly, fully missing rows are often discarded during preprocessing, which can lead to an underestimation of the achievable privacy amplification. In this sense, our results suggest revisiting standard preprocessing practices for differentially private pipelines, as retaining such data may help unlock additional privacy amplification prior to applying the algorithm.

On the other hand, when the differing record $i_{\ast}$ is always partially observed, our bound does not indicate any privacy amplification, as $\epsilon'=\epsilon$ and $\delta_{\Miss}(\epsilon')$ is upper bounded by $\delta$. More precisely, we prove in \Cref{app:proptight} that amplification is not possible for all randomized algorithms and all missing data processes when $\qstar =1$. 
 To move forward, we  consider in the next section more specific classes of $(\epsilon, \delta)$-DP algorithms in order to obtain more tangible privacy amplification.

\section{Privacy amplification of missing data for feature-wise Lipschitz queries}\label{sec:fwl}

We now refine the MAR (and MCAR) analysis for specific DP algorithms and show that privacy amplification can still occur even in the regime $\qstar = 1$, where the general result of Section~\ref{sec:amplif} does not yield improvements in $\epsilon$ or $\delta$. We focus on the classical Laplace and Gaussian mechanisms \citep{dwork2006calibrating,dwork2014algorithmic}, which add noise to the output of a numerical query. These mechanisms are typically defined for functions $f$ acting on complete datasets in $\datadomain^n$, which does not directly accommodate missing data. To address this, we consider mechanisms of the form $f_I = f \circ I$, where $f : \datadomain^n \rightarrow \mathbb{R}^k$ is a standard query and $I : \missdatadomain^n \rightarrow \datadomain^n$ is a data imputation procedure, as commonly used in the missing data literature. Throughout this section, we consider imputation by zero, in which $I$ replaces all missing values by zero while leaving all observed values unchanged. Our framework can be extended to handle more complex data-independent deterministic imputation, such as general constant imputations, or even conditional iterative imputations.
In this context, by controlling the regularity of the composed query $f_I$ through its sensitivity, and calibrating the parameters of the Laplace or Gaussian noise accordingly, we obtain differentially private mechanisms whose privacy guarantees are amplified by the inherent missing data process.

\subsection{Feature-Wise Lipschitz queries}\label{sec:fwlqueries}

In order to quantify the smoothness of a function, 
we introduce the class of feature-wise Lipschitz (FWL) queries as presented in Definition~\ref{def:fwl} below. 
\smallskip 

\begin{definition}
\label{def:fwl}
  Let $\query$ be a query that takes as input a dataset $\datamask$ in $\missdatadomain^n$ and outputs a vector in $\mathbb{R}^{k}$. The query $\query$ is FWL with respect to a given norm $\|.\|$ if there exist $L_1, \dots, L_d \in \mathbb{R}_{+}$ such that for any $\datamask \simeq \datamask'$ in $\missdatadomain^n$ differing in one record $i_{\ast}$, we have \vspace{-15pt}
  
  $$
    \| \query(\datamask) - \query(\datamask')\| \leq \sum_{j=1}^d L_j |I(\datamask)_{\istar}^{(j)} - I(\datamask')_{\istar}^{(j)}|,
  $$
 where $I(\datamask)_{\istar}^{(j)}$ is the $j$th feature of the $\istar$th row of the imputed dataset $I(\datamask)$.
\end{definition}

Although it may appear restrictive at first, the class FWL encompasses many natural queries arising in differentially private analysis. In particular, 
a wide range of queries (including linear, $\ell_1$-Lipschitz, and quadratic queries on bounded domains) fall within this class. Moreover, closure under linear combinations and Lipschitz post-processing enables the design of more complex DP pipelines from FWL queries. Such post-processing operations include, among others, linear projections, feature selection, and low-dimensional embeddings implemented via Lipschitz maps.\footnote{Refer to Appendix~\ref{app:FWL} for a list of queries extracted from the DP literature and proof they fall under the FWL class.} The main reason why FWL queries are useful to our analysis is because the impact on their output of changing a single sample can be decomposed additively across its input features. Define the $\ell_p$ sensitivity of a query $\query$ on incomplete datasets as
\begin{align*}
  \querysens{p} \eqdef \sup_{\datamask\simeq\datamask' \in \missdatadomain^n} \|\query(\datamask) -\query(\datamask') \|_p.
\end{align*} 
Recalling that $\datadomain \subset [-1,1]^{d}$, if the query $\query$ is FWL with respect to $\|.\|_p$, then its $\ell_p$ sensitivity on incomplete datasets admits a coordinate-wise decomposition of the form
\begin{align}
  \querysens{p} \leq 2 \sum_{j=1}^d L_j \eqdef C_p.
\label{eq_definition_upper_bound_sensitivity}
\end{align}


\subsection{Bounding the sensitivity of a FWL query under $\rho$-bounded observation}

To leverage the structural properties of FWL queries, we focus on a subclass of missing data processes in which at most a fraction $\rho$ of the features are observed. While any missing data process can be cast into this class by choosing $\rho$ appropriately, we instead fix $\rho$ a priori and study all mechanisms that satisfy this constraint for a given value of $\rho$. 

\begin{assumption}
  \label{ass:proportion_observed_entries}
  There exists $\rho > 0$
  such that, for any $\data \in \datadomain^n$ and $\mask \in  \textrm{supp}({\datamiss(\data)})
  \eqdef
  \left\{
  \mask \in \{0,1\}^{n\times d}
  :\Prob{\datamiss(\data)=\mask} > 0
  \right\}$,
  we have $\textrm{Card}(\operatorname{obs}(\masksample_i))/d \leq \rho,$ for all $i \in [n]$.
\end{assumption}

Echoing the discussion in Section~3, Appendix~\ref{app:rwdata} reports real-world missingness configurations showing settings in which \Cref{ass:proportion_observed_entries} holds. Furthermore, when \Cref{ass:proportion_observed_entries} does not hold uniformly, we further provide a relaxed version in Appendix~\ref{app:relaxed-fwl}, where the bounded-observation condition is allowed to fail on an exceptional set of masks with probability at most $\eta$. Additionally, we hereafter focus on data spaces $\datadomain \subset [-1,1]^{d}$. We only make this assumption for ease of presentation, but our results naturally extend to any $\ell_\infty$ bounded space. Under these assumptions, we can now present \Cref{lemma:sensmiss}, which is the cornerstone of the forthcoming amplification results. This lemma explicitly relates the sensitivity of a FWL query under missing data with the fraction of observed features $\rho$. The proof is deferred to Appendix~\ref{app:sensmiss}.

\smallskip 

\begin{restatable}{lemma}{Propsensitivity}\label{lemma:sensmiss} Let $\datadomain \subset [-1,1]^d$ and let $\datamiss$ be a MAR missing data process verifying \Cref{ass:proportion_observed_entries}. Let $\query: \missdatadomain^n \rightarrow \mathbb{R}^k$ be a FWL query with respect to  $\|.\|_p$ with constants $L_1, \dots, L_d$. We denote by $L_{(1)}, \dots, L_{(d)}$ the constants sorted by non-increasing order. Then, the $\ell_p$ sensitivity of $\query$ is such that 
\begin{align*}
    \querysensmiss{p}
    =
    \sup_{\substack{
\tilde\data(\mask)\simeq \tilde\data'(\mask') \\ 
\mask\in\mathrm{supp}(\datamiss(\data)) \\
\mask'\in\mathrm{supp}(\datamiss(\data'))
}}
    \left\|
        \query\big(\tilde{\data}(\mask)\big)
        -
        \query\big(\tilde{\data}'(\mask')\big)
    \right\|_p
    \leq
    2 \sum_{j=1}^{\lfloor\rho d \rfloor} L_{(j)}
    \eqdef \tilde{C}_p.
\end{align*}
\end{restatable}
Mechanically, only the observed features contribute to the query sensitivity, which translates into a smaller sensitivity of a query that operates on missing values: when all $L_j$ are equal, the ratio of sensitivity is controlled by $\tilde{C}_p/C_p \simeq \rho$. This will translate into stronger privacy guarantees in Sections \ref{sec:ampliffwl} and \ref{sec:ampliffwl2} where we analyze the Laplace and the Gaussian mechanisms respectively.  

\subsection{Missingness amplifies Laplace mechanism privacy}\label{sec:ampliffwl}


The Laplace mechanism is one of the most standard randomized algorithms to satisfy differential privacy. For any query $\query: \missdatadomain^n \rightarrow \mathbb{R}^k$, the Laplace mechanism (see, e.g.,~\cite{dwork2014algorithmic}) with scale parameter $b \geq 0$ takes as input any dataset $\datamask$ in $\missdatadomain^n$ and outputs a vector in $\mathbb{R}^k$ as
\begin{align*}
  \lap(\datamask) = \query(\datamask) + Y,
\end{align*}
where $Y = (Y_1, \dots, Y_k)$ is a random vector with independent identically distributed components $(Y_i)_{i\in[k]}$ from a Laplace distribution with scale $b$.  For any $\epsilon>0$, the Laplace mechanism can be shown to satisfy $(\epsilon, 0)$-DP as soon as $b \geq \querysensmiss{1}/\epsilon$. In  particular, if $\query$ is FWL with respect to $\|.\|_1$, the natural choice for the scale parameter is $b= C_1 / \epsilon$, where $C_1$ is defined as per~\eqref{eq_definition_upper_bound_sensitivity}. For this mechanism, we can show the following amplification result.


\begin{restatable}{theorem}{TheoremFWLLap}\label{thm:fwllaplace}
  Let $\datadomain \subset [-1,1]^d$, and let $\epsilon >0$. 
  Let $\mathcal{D}$ be a MAR missing data process verifying \Cref{ass:proportion_observed_entries}. Let $\query$ be a FWL query with respect to $\|.\|_1$, and $\lap$ be the Laplace mechanism, associated to $\query$ with scale parameter $b=C_1 / \epsilon$. Let also $\tilde{C}_1$ be as defined in \Cref{lemma:sensmiss}. Then the algorithm $\lapmiss$ defined as per~(\ref{def:A_D}) with $\lap$ and $\datamiss$ is $(\epsilon', 0)$-DP with \vspace{-5pt}
  
  $$ \epsilon' = \ln\!\big(1+p_{\ast} (e^{\tilde{\epsilon}}-1)\big), \quad \tilde{\epsilon} = (\tilde{C}_1/C_1)\epsilon. $$
%
\end{restatable}

 As in~\eqref{eq:ampMAR}, the bound on the algorithm privacy in \Cref{thm:fwllaplace} is controlled by the probability $p_{\ast}$.  However, by choosing an appropriate FWL query and a MAR missing data process, the bound also depends on the ratio of sensitivities $\tilde{C}_1 / C_1$, which can be very small when the fraction of observed feature $\rho$ is small. In particular, if all $ L_j$ are equal and $\rho d \in \mathbb{N}$, then $\tilde{C}_1/ C_1 = \rho$.  Noticing that $\ln \,\!(1+p_{\ast} (e^{\tilde{\epsilon}}-1)) \leq \tilde{\epsilon}$, 
\Cref{thm:fwllaplace} proves that, for any MAR missing process satisfying Assumption~\ref{ass:proportion_observed_entries}, $\lapmiss$ is $(\epsilon',0)$-DP with $\epsilon' \leq  \min\left((e-1) p_{\ast}, 1\right) \rho \epsilon.$ This privacy amplification phenomenon is quantified precisely in Appendix~\ref{corr:fwllaplace}.

\subsection{Missingness amplifies Gaussian mechanism privacy}\label{sec:ampliffwl2}

We now turn our focus towards the Gaussian mechanism. 
For any query $\query: \missdatadomain^n \rightarrow \mathbb{R}^k$, 
the Gaussian mechanism (see, e.g.,~\cite{dwork2014algorithmic}) with scale
parameter $\sigma > 0$ takes as an input any dataset $\tilde{\data} \in \missdatadomain^n$ and outputs a vector in $\mathbb{R}^k$ as
\begin{align}
  \gauss(\datamask) = \query(\datamask) + Y, 
\end{align}
where $Y$ is a Gaussian random variable $\mathcal{N}(0, \sigma^2I_k)$. For any $\epsilon \in (0, 1]$ and $\delta \in (0,1]$, the Gaussian mechanism can be shown to satisfy $(\epsilon, \delta)$-DP as soon as $\sigma \geq c \querysens{2} / \epsilon$, and any fixed $c > \sqrt{2 \operatorname{ln}(1.25/\delta)}$. In particular, if $\query$ is FWL with respect to $\|.\|_2$, the natural choice for the scale parameter is $\sigma = C_2 c/\epsilon$, where $C_2$ is defined as per~\eqref{eq_definition_upper_bound_sensitivity}. For this mechanism, we can show the following practical amplification result.

\begin{restatable}{theorem}{GaussianAmplif}\label{thm:gaussamplif}
   Let $\datadomain \subset [-1,1]^d$, and let $\epsilon \in (0,1]$ and $\delta \in (0,1]$. Let $\mathcal{D}$ be a MAR missing data process verifying \Cref{ass:proportion_observed_entries}. Let $\query$ be a FWL query with respect to $\|.\|_2$, and $\gauss$ be the  Gaussian mechanism, associated to $\query$ with scale parameter $\sigma = c \, C_2 / \epsilon$. 
  Then the algorithm $\guassmiss$ defined as per~(\ref{def:A_D}) with $\gauss$ and $\datamiss$ is $(\epsilon', \delta')$-DP with
%
 $$ \epsilon' = \ln\!\big(1+p_{\ast} (e^{\tilde{\epsilon}}-1)\big), \quad \tilde{\epsilon} = (\tilde{C}_2/C_2)\epsilon, \quad \text{and} \quad \delta'=\qstar\delta.$$
 \end{restatable}
\Cref{thm:gaussamplif} establishes privacy amplification for MAR missing mechanisms yielding the same comments as  \Cref{thm:fwllaplace}. The proof of \Cref{thm:gaussamplif} can be found in \Cref{app:gaussamplif} and is based on similar arguments as that of \Cref{thm:fwllaplace}. Again, privacy amplification occurs for different values of $p_{\ast}$, including $p_{\ast}=1$. Contrary to \Cref{thm:fwllaplace}, privacy amplification may occur for both $\epsilon$ and $\delta$. \Cref{thm:fwllaplace} and \Cref{thm:gaussamplif} show that privacy amplification occurs for some MAR missing data processes and the most standard differentially private algorithms. 

\vspace{-1ex}
\section{Limitations \& future work}\label{sec:discussion}
\vspace{-1ex}
Firstly, our analysis assumes that the missingness process acts independently across samples. This assumption is standard in the missing-data literature and is naturally aligned with the usual i.i.d. modeling of datasets. It is also technically central to our proofs: independence allows us to separate the mask of the differing record from the residual missingness pattern on the remaining records. Without this structure, the joint mask distribution could introduce dependencies across individuals, and these dependencies may themselves reveal information about the dataset. In such settings, missingness may no longer amplify privacy and could even lead to privacy degradation. To illustrate this phenomenon, we provide in Appendix~\ref{app:indepmask} a simple construction where correlations in the missingness patterns across samples induce additional leakage. An interesting direction for future work is to precisely understand the phenomena at work with non i.i.d. data that can affect privacy guarantees.

Secondly, our results emphasize the importance of understanding the missingness process from a privacy perspective. While identifying whether data is MCAR, MAR, or MNAR is classically challenging, this distinction is not merely of statistical interest in our setting: it directly impacts the form and strength of the resulting privacy guarantees. This connection provides compelling motivation for developing statistical procedures to assess or approximate the underlying missingness structure. 

\bibliographystyle{plainnat}
\bibliography{example_paper}


\newpage

\appendix

\section{Tools \& additional definitions}
In this section, we introduce several technical tools and additional definitions that are used in the proofs of the results presented in the main paper.

\subsection{substitute one record relation}\label{app:dist}

For all datasets $\datamask, \datamask' \in \missdatadomain^n$, we define the
\emph{substitute}-one distance by:
\begin{equation*}\label{eq:ds}
 d(\datamask,\datamask')
 \;=\;
 \min_{\pi \in \Gamma_n}\; \sum_{i=1}^n \mathds{1}\!\left\{\, \samplemask_i \neq \samplemask'_{\pi(i)} \right\},
\end{equation*}
where $\Gamma_n$ is the set of permutations on $[n]$. 
Two datasets $\datamask$ and $\datamask'$ are neighboring with respect to $\simeq$
if they differ in at most one element:
\begin{equation*}
 \datamask \simeq \datamask'
 \;\Leftrightarrow\;
 d(\datamask,\datamask') \le 1.
\end{equation*}

\subsection{$\alpha$-divergence}
As defined in \cite{chernoff1952measure}, the $\alpha$-divergence has several useful properties.
\begin{mdframed}
\begin{restatable}{proposition}{PropConv}\label{prop:conv}
  For any random variables $X_0, X_1, X_2$ taking values in $(E, \mathcal{E})$ and any $\alpha>0, \beta \in [0,1]$, we have
  \begin{equation*}
  \begin{aligned}
    & D_{\alpha}(X_0 || (1-\beta)X_1 + \beta X_2) \\&\leq (1-\beta) D_{\alpha}(X_0 || X_1) + \beta D_{\alpha}(X_0 || X_2)
  \end{aligned}
  \end{equation*}
\end{restatable}
\end{mdframed}
\begin{proof}
    Let us fix $S \in \mathcal{E}$.
  \begin{align*}
   &\Prob{X_0 \in S} - \alpha \big( (1-\beta)\Prob{X_1 \in S} + \beta\Prob{X_2 \in S} \big) \
   \\& =\Prob{X_0 \in S} - \beta \Prob{X_0 \in S} + \beta \Prob{X_0 \in S} - \alpha(1-\beta)\Prob{X_1 \in S} - \alpha\beta\Prob{X_2 \in S} \\ &= (1-\beta)\big(\Prob{X_0 \in S} - \alpha \Prob{X_1 \in S} \big) + \beta \big(\Prob{X_0 \in S} - \alpha \Prob{X_2 \in S} \big)
  \end{align*}

Since this holds for any $S \in \mathcal{E}$ we take the supremum on both sides of the equality and we conclude by using the definition of $D_{\alpha}$.
\end{proof}

\begin{mdframed}
\begin{proposition}{\citep{balle2018privacy}}\label{prop:advancedconv}
Let $X, X'$ be two random variables taking values in $(E, \mathcal{E})$. Assume there exist random variables $X_0, X_1,$ and $X_2 $ taking values in $(E, \mathcal{E})$ such that for all $S \in \mathcal{E}$, we have 
\begin{equation*}
  \Prob{X \in S} =(1-\eta)\Prob{X_0 \in S} + \eta \Prob{X_1 \in S} ~~\text{ and }~~ \Prob{X' \in S} =(1-\eta)\Prob{X_0 \in S} + \eta \Prob{X_2 \in S}, 
\end{equation*}
with $\eta \in [0,1]$.
For all $\alpha \geq 1$, let $\alpha' = 1 + \eta(\alpha -1)$ and $\beta = \frac{\alpha'}{\alpha}$. Then,
\begin{equation*}
   D_{\alpha'}(X|| X') = \eta D_{\alpha}(X_1 || (1-\beta)X_0 + \beta X_2).
\end{equation*}
\end{proposition}
\end{mdframed}

\begin{mdframed}
    \begin{proposition}\label{prop:alphatv}
        Let $X, Y$ and $Z$ three random variables taking values in $(E, \mathcal{E})$. For all $\alpha >0$, we have
        \begin{align}
            D_{\alpha}(X \|Y) \leq D_{\alpha}(X\|Z) + \alpha \mathrm{TV}(Z, Y).
        \end{align}
    \end{proposition}
\end{mdframed}
\begin{proof}
    Let us fix $S$ in $\mathcal{E}$.
    We have:
    \begin{align*}
        \Prob{X\in S} - \alpha \Prob{Y \in S} = \left(\Prob{X \in S} - \alpha \Prob{Z \in S}\right) + \alpha \left(\Prob{Z \in S} - \Prob{Y \in S}\right).
    \end{align*}
    By taking the supremum over all $S \in \mathcal{E}$, we obtain:
    \begin{align*}
        D_{\alpha}(X \| Y) &\leq D_{\alpha}(X \| Z) + \alpha \sup_{S \in \mathcal{E}}\left(\Prob{Z \in S}- \Prob{Y \in S} \right) \\
        &\leq  D_{\alpha}(X \| Z) + \alpha \mathrm{TV}(Z, Y),
    \end{align*}
    where $\mathrm{TV}(Z,Y) = \sup_{S \in \mathcal{E}}|\Prob{Z \in S}- \Prob{Y \in S}|$
\end{proof}

\subsection{Couplings}

Couplings provide a constructive tool that allows us to bound the divergence between mixtures by aligning the randomness in their components.

\begin{definition}\label{def:coupling}
  Let $\nu$ and $\nu'$ be probability measures on a measurable space $(E, \mathcal{E})$. A coupling $\pi$ of $\nu$ and $\nu'$ is a probability measure on $(E\times E, \mathcal{E} \otimes \mathcal{E})$ such that the marginals of $\pi$ are $\nu$ and $\nu'$ respectively, i.e.
  \begin{equation}
  \begin{aligned}
    \forall S \in \mathcal{E}, \quad 
    & \pi(S\times E) = \nu(S), \\ \text{and} \quad &\pi(E \times S) = \nu'(S).
  \end{aligned}
  \end{equation}
  We denote by $\mathcal{C}(\nu, \nu')$ the set of all couplings of $\nu$ and $\nu'$ on $(E, \mathcal{E})$.
\end{definition}

\subsection{Residual missing data process}

We introduce the following notation for later use in several proofs.
Since the missing data process is independent per sample, any data mask $\mask$ in $\{0,1\}^{n\times d}$, can be rewritten, up to a permutation as
\begin{align*}
  \mask = (\masksample_{i_{\ast}}, \res),
\end{align*}
where $\masksample_{\istar}$ denotes the mask corresponding to the differing sample $\istar$ and
\begin{align}
  \res = (\masksample_i)_{1 \leq i\neq i_{\ast} \leq n} \in \{0,1\}^{(n-1) \times d} \label{eq:defr}
\end{align}
collects the mask components for all remaining samples. We also introduce the residual dataset associated to $\data$ in $\datadomain^n$ as
\begin{align*}
  \data_{\text{res}} = (\sample_i)_{1 \leq i\neq i_{\ast} \leq n} \in \datadomain^{n-1}
\end{align*}
We refer to the induced mechanism denoted $\mathcal{R}$ as the residual missing data process of $\datamiss$.
\vspace{2ex}
\begin{mdframed}
\begin{definition}\label{def:res}
  Let $\datamiss$ be the missing data process.
  The residual missing data process $\mathcal{R}$ for $\datamiss$ takes as input a dataset in  $\datadomain^{n-1}$ and outputs a random vector taking values in $\{0,1\}^{(n-1) \times d}$. Formally, for all $\big(\masksample_1, \ldots, \masksample_{\istar}, \ldots, \masksample_n \big)^\top \hspace{-2pt}\in \{0,1\}^{n \times d}$, for any dataset $\data$ in $\datadomain^n$, we have
  \begin{equation*}
  \Prob{\mathcal{R}(\data_{\text{res}}) = \res} = \prod_{i\neq i_{\ast}}^{n}\Prob{\featuremiss(\sample_{i}) = \masksample_{i}},
\end{equation*}
where $\res$ is defined in \cref{eq:defr}.
\end{definition}
\end{mdframed}

\subsection{MAR and $\Hstar^c$}

\vspace{3ex}
\begin{mdframed}
  \begin{restatable}{lemma}{LemmaMAR}\label{lemma:pstar}
  Let $\Hstar(\cdot,\cdot)$ be as defined in \cref{eq:defhc}.
  If the missing data process $\datamiss$ is MAR, for any neighboring datasets $\data \simeq \data'$ in $\datadomain^n$ and for all $\mask$ in $\Hstar^c(\data, \data')$ such that the differing sample is completely masked, we have
  \begin{align*}
    \Prob{\datamiss(\data) = \mask} = \Prob{\datamiss(\data') = \mask}.
  \end{align*}
\end{restatable}
\end{mdframed}

\begin{proof}
  Let us fix two neighboring datasets $\data\sim\data'$ in $\datadomain^n$, let $\istar$ be the index of the differing row between $\data$ and $\data'$ and $\mask$ in $\Hstar^c(\data, \data')$.

By \cref{def:missingmech}, we have
\begin{align*}
  \Prob{\datamiss(\data) = \mask} = \prod_{i=1}^n \Prob{\featuremiss\big(\sample_{i}\big) = \masksample_i}
  = \Prob{\featuremiss(\sample_{\istar}) = \masksample_{i_\ast}}. \Big(\prod_{ 1 \leq i\neq i_{\ast} \leq n} \Prob{\featuremiss(\sample_{i}) = \masksample_i} \Big).
\end{align*}

Note that, by definition of $\Hstar(\data, \data')$, we have $\masksample_{\istar} = \mathbf{1}_{\mathbb{R}^d}$ i.e., the observed parts of the $\sample_{\istar}$ and $\sample'_{\istar}$ are equal and both are empty, in particular $\sample_i^{\operatorname{obs}(\masksample_i)} = \sample'^{\,\operatorname{obs}(\masksample_i)}_i$ . Since $\datamiss$ is MAR, this implies that
\begin{align*}
  \Prob{\featuremiss(\sample_{\istar}) = \masksample_{\istar}} = \Prob{\featuremiss(\sample'_{\istar}) = \masksample_{\istar}}.
\end{align*}

In addition, for $i \neq \istar$ in $[n]$, since $\data\simeq\data'$, we have $\sample_i = \sample'_i$. Therefore for all $i \in [n]$ such that $i \neq \istar$, we have
\begin{align*}
  \Prob{\featuremiss(\sample_i) = \masksample_i} = \Prob{\featuremiss(\sample'_i) = \masksample_i}.
\end{align*}

Finally, reinjecting these observations in the first equality, we obtain
\begin{align*}
  \Prob{\datamiss(\data)= \mask} = \Prob{\datamiss(\data')= \mask}.
\end{align*}
\end{proof}

\section{Proofs of Section~\ref{sec:amplif}}

In this section, we provide sketches of the proofs for our main results, which will be incorporated into the main paper upon acceptance.

\subsection{Proof of Theorem~\ref{thm:MNAR}}\label{app:proof}
\vspace{3ex}
\begin{mdframed}
 \MNARThm*
\end{mdframed}

\begin{tcolorbox}[boxrule=0pt,standard jigsaw, opacityback=0, frame hidden,sharp corners,enhanced,breakable,borderline west={1pt}{0pt}{gray},left=6pt,right=0pt,top=0pt,bottom=0pt,boxsep=1pt]
\begin{proof}[Sketch of proof]
Let $\mathcal{E}$ be the set of all measurable output sets for $\Miss$. Then by construction of $\Miss$ and by the law of total probability, for any finite dataset $\data \in \datadomain^n$ and any event $S \in \mathcal{E}$, the probability distribution of the random variable $\Miss(\data)$ is a mixture. More specifically, the probability $\Prob{\Miss(\data) \in S}$ can be decomposed as
\begin{align} 
\Prob{\Miss(\data) \in S} = 
\sum_{\mask \in \{0,1\}^{\,n\times d}}
  \hspace{-10pt} \Prob{\datamiss(\data) = \mask}\,
   \Prob{\algo\big(\tilde{\data}(\mask)\big) \in S}. \label{eq:mixture}
\end{align}

We provide here the most important points to observe. Fix two neighboring datasets $\data \simeq \data'$.
  
 \textbf{A)} We first decompose the distribution of $\Miss(\data)$ (resp. $\Miss(\data')$) onto the events $\datamiss(\data) \in \Hstar(\data,\data')$ and $\datamiss(\data) \in \Hstar(\data,\data')^c$ (resp. $\datamiss(\data) \in \Hstar(\data,\data')$ and $\datamiss(\data) \in \Hstar(\data,\data')^c$). We show that on $\Hstar^c(\data,\data')$, the conditional laws coincide. 
 This allows us to decompose each distribution between a common overlapping component and a dataset-specific non-overlapping component.

 \textbf{B)} Note that, since $\datamiss$ is MNAR, the latter decompositions involve different mixture weights that depends on $\data$ and $\data'$. We introduce an auxiliary distribution that aligns the weights of $\Miss(\data)$ with the components of $\Miss(\data')$. Using \Cref{prop:alphatv}, we isolate two effects: a divergence term between aligned mixtures and a residual term capturing the mismatch in weights $e^{\epsilon'(\data \vert \data')}|p_{\data \mid \data'}-p_{\data' \mid \data}|$. 
  
  \textbf{C)} We decompose the previous divergence using its convexity properties (see \Cref{prop:advancedconv}). We simply need to bound the divergence between the non-overlapping parts of the two distributions and the specific part of either distribution and their shared overlap. 
  
  \textbf{D)} To bound the latter quantities, we design explicit couplings of the missing data process. These couplings are designed so that, on their support, the resulting incomplete datasets remain neighboring. This allows us to directly invoke the $(\epsilon,\delta)$-DP guarantee of $\algo$, yielding uniform bounds on all component-wise divergences. We prove
\begin{align}\label{eq:art1MNAR}
\Divmiss{e^{\epsilon'(\data\vert\data')}} \leq p_{\data \mid \data'} \delta + e^{\epsilon'(\data\vert\data')}|p_{\data \mid \data'}-p_{\data' \mid \data}|.
  \end{align} 

\textbf{E)} We design another coupling between the full mask distributions of $\datamiss(\data)$ and $\datamiss(\data')$ that operates at the level of the missingness process: it synchronizes the residual mask across all non-differing rows while sampling independently the masks of the differing row according to their dataset-dependent laws. We thus obtain
\begin{align}\label{eq:art2MNAR}
\Divmiss{e^{\epsilon}} \leq \delta.
\end{align}
\textbf{F)} Finally, we interpolate between the bounds \eqref{eq:art1MNAR} and \eqref{eq:art2MNAR} using the sub-convexity property of the $\alpha$-divergence.
\end{proof}
\end{tcolorbox}

\begin{proof} 
Let us fix $\epsilon > 0$, $\delta \in [0,1]$, a pair of neighboring datasets $\data\simeq\data'$ in $\datadomain^n$ and let $\istar$ be the index of their differing row.

To prove the Theorem, we proceed in three steps. 

\textbf{I)} We first prove that
\begin{align}\label{eq:wcast_dp}
\Divmiss{e^{\epsilon'(\data \vert \data')}} \leq p_{\data \mid \data'}\delta
+
e^{\epsilon'(\data \vert \data')}|p_{\data \mid \data'} - p_{\data' \mid \data}|
\end{align}
\textbf{II)} Then, we establish that
\begin{align}\label{eq:mnar_dp}
\Divmiss{e^{\epsilon}} \leq \delta,
\end{align}

and \textbf{III)} we conclude by introducing both $\epsilon(\lambda, \data \mid  \data') = \log(\lambda e^{\epsilon} + (1-\lambda)e^{\epsilon'(\data \vert \data')})$ for any $\lambda \in [0,1]$ and $\delta_{\data\vert\data'} = p_{\data \mid \data'}\delta
+
e^{\epsilon'(\data \vert \data')}|p_{\data \mid \data'} - p_{\data' \vert \data}|$ then invoke the properties of the $\alpha$-divergence. 

Let $\mathcal{E}$ be the set of all measurable output sets for $\algo_{\datamiss}$.

\textbf{I) Proof of \Cref{eq:wcast_dp}}

Let us consider $\data \simeq \data'$ in $\datadomain^n$ and fix $S$ an arbitrary event in $\mathcal{E}$. We recall that
\begin{align*}
\Prob{\Miss(\data) \in S} = &\sum_{\mask \in \{0,1\}^{n\times d}}
    \MargZ, \\
    \Prob{\Miss(\data') \in S} = &\sum_{\mask \in \{0,1\}^{n \times d}}
    \MargZz.
\end{align*}

Besides, by definition  $\Hstar(\data, \data') =\left\{\mask \in \{0,1 \}^{n \times d}\mid\masksample_{i^{\ast}} \neq 1_{\mathbb{R}^d}\right\}$. For conciseness, $\Hstar(\data, \data')$ will be denoted by $\Hstar$ in the rest of the proof. 
We have
\begin{align}
\Prob{\Miss(\data) \in S}
&= \sum_{\mask \in \Hstar^c}\MargZ + \sum_{\mask \in \Hstar}
    \MargZ \label{eq1mnar}, \\
\Prob{\Miss(\data') \in S}
&= \sum_{\mask \in \Hstar^c}
    \MargZz +\sum_{\mask \in \Hstar}
    \MargZz \label{eq2mnar}.
\end{align}


We now decompose the previous probability masses into the following four parts. 

\begin{itemize}
    \item[i)] We denote by $\mu^{(0)}_{\data}$ the probability measure defined for all $\mask \in \{0,1\}^{n\times d}$ by
     \begin{align*}
    \mu^{(0)}_{\data}(\mask) = \frac{1}{1-p_{\data \mid \data'}} \Ind{\mask\in \Hstar^c}\Prob{\datamiss(\data) = \mask}.
    \end{align*}
    Since $1 - p_{\data \mid \data'} = \Prob{\datamiss(\data) \in \Hstar^c}$, $\mu^{(0)}_{\data}$ is indeed a probability distribution.
    The corresponding proportion of probability mass in $\algo_{\datamiss}(\data)$ is defined by
    \begin{align}
        w^{(0)}_{\data}(S) = \sum_{\mask \in \{0,1\}^{n \times d}} \mu^{0}_{\data}(\mask) \Prob{\algo\left(\tilde{\data}(\mask)\right) \in S}. \label{eq:w0}
    \end{align}
    \item[ii)] Similarly we introduce the probability measure $\mu^{(0)}_{\data'}$ defined for all $\mask \in \{0,1\}^{n\times d}$ by
    \begin{align*}
    \mu^{(0)}_{\data'}(\mask) = \frac{1}{1-p_{\data' \mid \data}} \Ind{\mask\in \Hstar^c}\Prob{\datamiss(\data') = \mask},
    \end{align*}
    and its corresponding proportion of probability mass in $\algo_{\datamiss}(\data')$
    \begin{align*}
        w^{(0)}_{\data'}(S) = \sum_{\mask \in \{0,1\}^{n \times d}} \mu^{0}_{\data'}(\mask) \Prob{\algo\left(\tilde{\data}'(\mask)\right) \in S}. 
    \end{align*} 
    \item[iii)] Then, we define the probability measure defined for all $\mask \in \{0,1\}^{n\times d}$ by
    \begin{align*}
    \mu^{(1)}_{\data}(\mask) = \frac{1}{p_{\data \mid \data'}} \Ind{\mask\in \Hstar}\Prob{\datamiss(\data) = \mask}. 
  \end{align*}
The corresponding proportion of probability mass in $\algo_{\datamiss}(\data)$ is defined by
    \begin{align*}
    w^{(1)}_{\data}(S) = \sum_{\mask \in \{0,1\}^{n \times d}} \mu^{(1)}_{\data}(\mask) \Prob{\algo\left(\tilde{\data}(\mask)\right) \in S}.
    \end{align*}
    \item[iv)] Finally, we define the probability measure defined for all $\mask \in \{0,1\}^{n\times d}$ by
    \begin{align*}
    \mu^{(1)}_{\data'}(\mask) = \frac{1}{p_{\data' \mid \data}} \Ind{\mask\in \Hstar}\Prob{\datamiss(\data') = \mask}. 
  \end{align*}
  The corresponding proportion of probability mass in $\algo_{\datamiss}(\data')$ is defined by
    \begin{align*}
    w^{(1)}_{\data'}(S) = \sum_{\mask \in \{0,1\}^{n \times d}} \mu^{(1)}_{\data'}(\mask) \Prob{\algo\left(\tilde{\data'}(\mask)\right) \in S}.
    \end{align*}
\end{itemize}

With the above notations, we can rewrite \Cref{eq1mnar,eq2mnar} as 
\begin{align*}
    & \Prob{\Miss(\data) \in S} = (1-p_{\data \mid \data'})w^{(0)}_{\data}(S)  + p_{\data \mid \data'}w^{(1)}_{\data}(S), \\
     & \Prob{\Miss(\data') \in S} = (1-p_{\data' \mid \data})w^{(0)}_{\data'}(S)  + p_{\data' \mid \data}w^{(1)}_{\data'}(S).
\end{align*}

\medskip
\noindent\textbf{I.1) Equality of the conditional laws on $\Hstar^c$} 

Let $\mask\in \Hstar^c$. We use the decomposition with the residual mask $\res \in\{0,1\}^{(n-1)\times d}$, the residual dataset $\data_\text{res} \in \datadomain^{n-1}$ and the residual missing data process $\mathcal{R}$ introduced in Definition~\ref{def:res}
\begin{equation*}
  \Prob{\mathcal{R}(\data_{\text{res}}) = \res} = \prod_{i\neq i_{\ast}}^{n}\Prob{\featuremiss(\sample_{i}) = \masksample_{i}}.
\end{equation*}

Since, on $\Hstar^c$, the differing sample $i^*$ is fully masked, we can write
\begin{align*}
\mask = (\mathbf{1}_{\mathbb R^d},\res).
\end{align*}

Note also that the missingness process acts independently across samples, therefore
\begin{align*}
\Prob{\datamiss(\data)=\mask}
&=
\Prob{\featuremiss(\data_{i^*})=\mathbf{1}_{\mathbb R^d}}
\Prob{\mathcal{R}(\data_{\mathrm{res}})=\res}, \\
\Prob{\datamiss(\data')=\mask}
&=
\Prob{\featuremiss(\data'_{i^*})=\mathbf{1}_{\mathbb R^d}}
\Prob{\mathcal{R}(\data'_{\mathrm{res}})=\res}.
\end{align*}
By definition of $p_{\data \mid \data'}$ and $p_{\data' \mid \data}$, we have
\begin{align*}
\Prob{\mathcal{F}(\data_{i^*})=\mathbf{1}_{\mathbb R^d}} = 1-p_{\data \mid \data'},
\qquad
\Prob{\mathcal{F}(\data'_{i^*})=\mathbf{1}_{\mathbb R^d}} = 1-p_{\data' \mid \data}.
\end{align*}
Besides, since $\data$ and $\data'$ only differ on $i^*$, we have $
\data_{\mathrm{res}}=\data'_{\mathrm{res}}$, hence
\begin{align*}
\Prob{\mathcal{R}(\data_{\mathrm{res}})=r}
=
\Prob{\mathcal{R}(\data'_{\mathrm{res}})=r}.
\end{align*}

Therefore,
\begin{align*}
\mu^{(0)}_{\data}(\mask)=
\frac{1}{1-p_{\data \mid \data'}}\Prob{\datamiss(\data)=\mask} =
\frac{1}{1-p_{\data \mid \data'}}
(1-p_{\data \mid \data'})\Prob{\mathcal{R}(\data_{\mathrm{res}})=r} =
\Prob{\mathcal{R}(\data_{\mathrm{res}})=r},
\end{align*}
and similarly
\begin{align*}
\mu^{(0)}_{\data'}(\mask)=
\frac{1}{1-p_{\data' \mid \data}}\Prob{\datamiss(\data')=\mask} =
\frac{1}{1-p_{\data' \mid \data}}
(1-p_{\data' \mid \data})\Prob{\mathcal{R}(\data'_{\mathrm{res}})=r} =
\Prob{\mathcal{R}(\data'_{\mathrm{res}})=r}.
\end{align*}
Since $\data_{\mathrm{res}}=\data'_{\mathrm{res}}$, we conclude that for all $\mask\in\Hstar^c$,
\begin{align*}
\mu^{(0)}_{\data}(\mask)=\mu^{(0)}_{\data'}(\mask).
\end{align*}
As for all $\mask \in H_*$ we have $\mu^{(0)}_{\data}(\mask)=\mu^{(0)}_{\data'}(\mask)=0$, this also means that 
\begin{align*}
    \mu^{(0)}_{\data}=\mu^{(0)}_{\data'}.
\end{align*}

Finally, for every $\mask\in\Hstar^c$, the differing sample is fully masked, so
\begin{align*}
\tilde{\data}(\mask)=\tilde{\data}'(\mask).
\end{align*}
Using the equality of the conditional laws, we get for every measurable event $S\in\mathcal E$,
\begin{align*}
w^{(0)}_{\data}(S)
&=
\sum_{\mask}\mu^{(0)}_{\data}(\mask)\Prob{\algo(\tilde{\data}(\mask))\in S} \\
&=
\sum_{\mask}\mu^{(0)}_{\data'}(\mask)\Prob{\algo(\tilde{\data}'(\mask))\in S}
\\&=
w^{(0)}_{\data'}(S).
\end{align*}

Therefore, we introduce $w^{(0)}$ as the common mass between $\algo_{\data}(\data)$ and $\algo_{\data}(\data')$ defined as
\begin{align*}
    w^{(0)} \eqdef w^{(0)}_{\data}=w^{(0)}_{\data'}.
\end{align*}

\medskip
\noindent\textbf{I.2) Aligning the mixture weights via an auxiliary distribution}

To align the mixture weights between $\Miss(\data)$ and $\Miss(\data')$, we introduce an auxiliary distribution $P$. We define for all $S \in \mathcal{E}$, 
\begin{align*}
P(S)
= (1-p_{\data \mid \data'})w^{(0)}(S) + p_{\data \mid \data'}w^{(1)}_{\data'}(S).
\end{align*}

By Proposition~\ref{prop:alphatv}, with $\epsilon'(\data \vert \data')=\log(1+p_{\data \mid \data'}(e^\epsilon-1))$, we have
\begin{align}
D_{e^{\epsilon'(\data \vert \data')}}\left(\mathcal{A_{\datamiss}(\data)}\|\mathcal{A_{\datamiss}(\data')}\right)
\leq
D_{e^{\epsilon'(\data \vert \data')}}(\algo_{\datamiss}(\data)\|P)
+
e^{\epsilon'(\data \vert \data')}\mathrm{TV}(\algo_{\datamiss}(\data'),P).
\label{eq:key}
\end{align}
First, let us bound $\mathrm{TV}(\algo_{\datamiss}(\data'),P)$.

Since we prove that $w^{(0)} = w^{(0)}_{\data} = w^{(0)}_{\data'}$, we have
\begin{align*}
\Prob{\algo_{\datamiss}(\data') \in S} - P(S)
&= (p_{\data \mid \data'} - p_{\data' \mid \data})\big(w^{(1)}_{\data'}(S) - w^{(0)}(S)\big).
\end{align*}

Thus,
\begin{align}
\mathrm{TV}(\algo_{\datamiss}(\data'),P)
\leq |p_{\data \mid \data'} - p_{\data' \mid \data}|\text{TV}\big(w^{(1)}_{\data}, w^{(0)}\big) \leq |p_{\data \mid \data'} - p_{\data' \mid \data}|.\label{ineq:tvaux}
\end{align}

Then, let us bound $D_{e^{\epsilon'}}(\algo_{\datamiss}(\data)\|P)$.

We recall that
\begin{align*}
\Prob{\algo_{\datamiss}(\data) \in S} &= (1-p_{\data \mid \data'})w^{(0)}(S) + p_{\data \mid \data'}w^{(1)}_{\data}(S), \\
P(S) &= (1-p_{\data \mid \data'})w^{(0)}(S) + p_{\data \mid \data'}w^{(1)}_{\data'}(S).
\end{align*}

Applying Proposition~\ref{prop:advancedconv} with $\beta = e^{\epsilon'(\data \vert \data')-\epsilon}$
, we obtain
\begin{align*}
D_{e^{\epsilon'(\data \vert \data')}}(\algo_{\datamiss}(S)\|P)
=
p_{\data \mid \data'}
D_{e^\epsilon}\!\left(
w^{(1)}_{\data}
\;\middle\|\;
(1-\beta)w^{(0)} + \beta w^{(1)}_{\data'}
\right).
\end{align*}
By \Cref{prop:conv}, we get 
\begin{align}
D_{e^{\epsilon'(\data \vert \data')}}(\algo_{\datamiss}(S)\|P)
\le
p_{\data \mid \data'}\Big(
(1-\beta)D_{e^\epsilon}(w^{(1)}_{\data}\|w^{(0)})
+
\beta D_{e^\epsilon}(w^{(1)}_{\data}\|w^{(1)}_{\data'})
\Big). \label{ineq:b}
\end{align}

Next, we establish an upper bound for $D_{e^\epsilon}(w^{(1)}_{\data}\|w^{(0)})$ (cf. \textbf{I.3}) and $D_{e^\epsilon}(w^{(1)}_{\data}\|w^{(1)}_{\data'})$ (cf. \textbf{I.4}) using the coupling method.

\textbf{I.3) Bounding $D_{e^\epsilon}(w^{(1)}_{\data}\|w^{(0)})$}

Let us consider an arbitrary coupling $\pi \in \mathcal{C}(\mu^{(0)}_{\data},\mu^{(1)}_{\data})$ (see \eqref{def:coupling} for more details). By definition, it verifies
\begin{align}
  &\forall \mask' \in \{0,1\}^{n \times d}, \sum_{\mask \in \{0,1\}^{n \times d}} \pi(\mask, \mask') = \mu^{(1)}_{\data}(\mask'),\label{eq:temp} \\
  \text{and} \quad &\forall \mask \in \{0,1\}^{n\times d},\sum_{\mask' \in \{0,1\}^{n \times d}} \pi(\mask, \mask') = \mu^{(0)}_{\data}(\mask). \label{eq:temp2}
\end{align}

Therefore, we can write
\begin{align*}
  w^{(1)}_{\data}(S) - e^{\epsilon}w^{(0)}_{\data}(S) &=\sum_{\Biin} \mu^{(1)}_{\data}(\mask')\Prob{\algo(\tilde{\data}(\mask')) \in S} - e^{\epsilon} \hspace{-5pt}\sum_{\Bin} \mu^{(0)}_{\data'}(\mask)\Pmxx{S}. \\
  \intertext{Furthermore, using \eqref{eq:temp} and \eqref{eq:temp2} we get 
  }
  w^{(1)}_{\data}(S) - e^{\epsilon}w^{(0)}_{\data}(S) & =\sum_{\mask \in \{0,1\}^{n\times d}}\sum_{\mask' \in \{0,1\}^{n\times d}} \pi(\mask,\mask')\big(\Prob{\algo(\tilde{\data}(\mask)) \in S} -e^{\epsilon}\Prob{\algo(\tilde{\data}'(\mask)) \in S} \big).
\end{align*}

Hence by taking the supremum over all $S \in \mathcal{E}$, we have
\begin{align}\label{eq:div}
  &D_{e^\epsilon}(w^{(1)}_{\data}\|w^{(0)}_{\data}) \leq \sum_{\mask, \mask' \in \{0,1\}^{n\times d}} \pi(\mask,\mask')
   D_{e^{\epsilon}}(\algo(\tilde{\data}(\mask))
   || \algo(\tilde{\data}'(\mask'))).
\end{align}

Since the missing data process is independent per sample, for any data mask $\mask$ in $\{0,1\}^{n\times d}$, we use the decomposition with the residual mask $\res$, the residual dataset $\data_\text{res}$ and the residual missing data process $\mathcal{R}$ introduced in Definition~\ref{def:res}
\begin{equation*}
  \Prob{\mathcal{R}(\data_{\text{res}}) = \res} = \prod_{i\neq i_{\ast}}^{n}\Prob{\featuremiss(\sample_{i}) = \masksample_{i}}.
\end{equation*}
Using this, we can rewrite for any $\data \in \dataset$ and $\mask \in \{0,1\}^{n \times d}$
\begin{align}
  \mu^{(1)}_{\data}(\mask)
  &= \frac{1}{p_{\data \mid \data'}} \Ind{\mask \in \Hstar}\Pfeat{\data} \nonumber \\
  &= \frac{1}{p_{\data \mid \data'}}\Ind{(\bistar, \res) \in\Hstar}
  \Prob{\featuremiss(\sample_{i_\ast}) = \bistar}
  \Prob{\mathcal{R}(\data_{\mathrm{res}}) = \res}, \label{eq:mu1def} \\
  \mu^{(0)}_{\data}(\mask)
  &= \frac{1}{1-p_{\data \mid \data'}} \Ind{\mask \in \Hstar^c}
  \Prob{\datamiss(\data) = \mask} \nonumber \\
  &= \frac{1}{1-p_{\data \mid \data'}}\Ind{(\bistar, \res) \in \Hstar^c}
  \Prob{\featuremiss(\sample_{i_\ast}) = \bistar}
  \Prob{\mathcal{R}(\data_{\mathrm{res}}) = \res} \nonumber \\
  &= \frac{1}{1-p_{\data \mid \data'}}
  \Prob{\featuremiss(\sample_{i_\ast}) = \mathbf 1_d}
  \Prob{\mathcal{R}(\data_{\mathrm{res}}) = \res}
  \quad
  \text{(since } \bistar=\mathbf 1_d \text{ on } \Hstar^c \text{)} \nonumber \\
  &= \frac{1}{1-p_{\data \mid \data'}}
  (1-p_{\data \mid \data'})
  \Prob{\mathcal{R}(\data_{\mathrm{res}}) = \res}
  \quad
  \text{(since } \Prob{\featuremiss(\sample_{i_\ast}) = \mathbf 1_d}=1-p_{\data \mid \data'} \text{)} \nonumber \\
  &= \Prob{\mathcal{R}(\data_{\mathrm{res}}) = \res}. \nonumber
\end{align}
Intuitively, we want to build a coupling procedure $\pi_1$ such that on its support $\tilde{\data}(\mask) \simeq \tilde{\data}'(\mask')$. To achieve this, we proceed as follow:
\begin{itemize}
  \item we sample a mask $\mask = (\masksample_{i_\ast}, \res)$ from $\mu^{(1)}_{\data}$;
  \item we define $\mask' = (\mathbf{1}_{\mathbb{R}^d}, \res)$.
\end{itemize}
Formally, let us define for all $\mask, \mask' \in \{0,1\}^{n \times d}$
\begin{align}
  &\pi_{1}(\mask, \mask') =\pi_{1}((\bistar, \res), (\bistar', \res')) = \mu^{(1)}_{\data}((\bistar,\res))\Ind{\{\res=\res'\}\cap\{\bistar' =\mathbf{1}_{\mathbb{R}^d}\}} \label{eq:piopt}.
\end{align}
We show that $\pi_{1} \in \mathcal{C}(\mu_{\data}^{(1)}, \mu^{(0)}_{\data})$. For any $\mask \in \{0,1\}^{n \times d}$, we have
\begin{align*}
  \sum_{\mask' \in \{0,1\}^{n \times d}} \pi_1(\mask, \mask') &= \sum_{(\bistar', \res') \in \{0,1\}^{n \times d}} \pi_1((\bistar, \res), (\bistar', \res'))
  \\&= \pi_1((\bistar, \res), (\mathbf{1}_{\mathbb{R}^d}, \res)) \quad \text{(by \cref{eq:piopt})} \\
  &= \mu^{(1)}_{\data}((\bistar, \res)) \\
  &= \mu^{(1)}_{\data}(\mask)
  \end{align*}
Furthermore, for any $\mask' \in \{0,1\}^{n \times d}$ we have
\begin{align*}
  \sum_{\mask \in \{0,1\}^{n \times d}} \pi_1(\mask, \mask')
  &= \sum_{(\bistar, \res) \in \{0,1\}^{n \times d}}
  \pi_1((\bistar, \res), (\bistar', \res')) \\
  &= \sum_{(\bistar, \res) \in \{0,1\}^{n \times d}}
  \mu^{(1)}_{\data}((\bistar, \res))
  \Ind{\{\res=\res'\}\cap\{\bistar'=\mathbf 1_d\}} \\
  &= \sum_{\bistar \in \{0,1\}^d}
  \mu^{(1)}_{\data}((\bistar, \res'))
  \Ind{\{\bistar'=\mathbf 1_d\}} \\
  &= \sum_{\bistar \neq \mathbf 1_d}
  \mu^{(1)}_{\data}((\bistar, \res'))
  \Ind{\{\bistar'=\mathbf 1_d\}}
  \quad
  \text{(by definition of } \mu^{(1)}_{\data} \text{ in \cref{eq:mu1def})} \\
  &= \frac{1}{p_{\data \mid \data'}}
  \Prob{\mathcal{R}(\data_{\mathrm{res}})= \res'}
  \sum_{\bistar \neq \mathbf 1_d}
  \Prob{\featuremiss(\sample_{i_\ast})=\bistar}
  \Ind{\{\bistar'=\mathbf 1_d\}} \\
  &= \Prob{\mathcal{R}(\data_{\mathrm{res}})= \res'}
  \Ind{\{\bistar'=\mathbf 1_d\}}
  \quad
  \text{since }
  p_{\data \mid \data'}
  =
  \sum_{\bistar \neq \mathbf 1_d}
  \Prob{\featuremiss(\sample_{i_\ast})=\bistar} \\
  &= \mu^{(0)}_{\data}(\mask').
\end{align*}
Therefore, by substituting ~\eqref{eq:div} with $\pi = \pi_1$, we obtain
\begin{align*}
    D_{e^{\epsilon}}(w^{(1)}_{\data} || w^{(0)}_{\data}) \leq \sum_{\mask, \mask' \in \{0,1\}^{n\times d}} \pi_1(\mask,\mask')
   D_{e^{\epsilon}}(\algo(\tilde{\data}(\mask))
   || \algo(\tilde{\data}'(\mask'))).
\end{align*}

Besides, by construction for any $\mask, \mask' \in \{0,1\}^{n \times d},$ $\pi_1(\mask, \mask') > 0$ iif $\tilde{\data}(\mask) \simeq  \tilde{\data'}(\mask')$. Hence using the fact that $\algo$ is $(\epsilon, \delta)$-DP, we get
 \begin{align}
  & D_{e^{\epsilon}}(w^{(1)}_{\data} || w^{(0)}_{\data}) \leq \sum_{\mask, \mask'\in \{0,1\}^{n\times d}} \pi_1(\mask,\mask')
   D_{e^{\epsilon}}(\algo(\tilde{\data}(\mask))
   || \algo(\tilde{\data}'(\mask'))) \leq \delta_{\algo}(\epsilon) \leq \delta. \label{ineq:b1}
\end{align}

\textbf{I.4) Bounding $D_{e^\epsilon}(w^{(1)}_{\data}\|w^{(1)}_{\data'})$}

Let us consider a coupling $\pi \in \mathcal{C}(\mu^{(1)}_{\data},\mu^{(1)}_{\data'})$ (see \eqref{def:coupling}).
By definition, it verifies
\begin{align}
  &\forall \mask' \in \{0,1\}^{n \times d}, \sum_{\mask \in \{0,1\}^{n \times d}} \pi(\mask, \mask') = \mu^{(1)}_{\data'}(\mask'), \\
  &\forall \mask \in \{0,1\}^{n \times d}, \sum_{\mask' \in \{0,1\}^{n \times d}} \pi(\mask, \mask') = \mu^{(1)}_{\data}(\mask).
\end{align}

Similarly, to \eqref{eq:div} we have
\begin{align}\label{eq:div_w1_w1p}
  D_{e^{\epsilon}}(w^{(1)}_{\data} || w^{(1)}_{\data'})
  &\leq \sum_{\mask, \mask' \in \{0,1\}^{n\times d}} \pi(\mask,\mask')
   D_{e^{\epsilon}}(\algo(\tilde{\data}(\mask))
   || \algo(\tilde{\data}'(\mask'))).
\end{align}

We now build an explicit coupling $\pi_{2}$ such that on its support $\tilde{\data}(\mask) \simeq \tilde{\data}'(\mask')$.
Since the missing data process is independent per sample, for any data mask $\mask$ in $\{0,1\}^{n\times d}$,
we also use the decomposition $\mask=(\bistar,\res)$ introduced in Definition~\ref{def:res},
where $\bistar$ is the mask of the $\istar$-th sample and $\res$ the residual mask.

We recall that, for any $\data\in\dataset$ and any $\mask\in\Hstar$,
\begin{align*}
  \mu^{(1)}_{\data}(\mask) = \mu^{(1)}_{\data}(\bistar,\res)
  &= \frac{1}{p_{\data \mid \data'}}\Ind{(\bistar,\res)\in\Hstar}
  \Prob{\featuremiss(\sample_{\istar})=\bistar}\Prob{\mathcal{R}(\data_{\text{res}}) = \res},\\
  \mu^{(1)}_{\data}(\mask') =\mu^{(1)}_{\data'}(\bistar',\res)
  &= \frac{1}{p_{\data' \mid \data}}\Ind{(\bistar',\res)\in\Hstar}
  \Prob{\featuremiss(\sample'_{\istar})=\bistar'}\Prob{\mathcal{R}(\data_{\text{res}}) = \res}.
\end{align*}

To build $\pi_{2}$, we proceed as follows:
\vspace{-2ex}
\begin{itemize}
  \item we sample a mask $\mask=(\bistar,\res)$ from $\mu^{(1)}_{\data}$;
  \item fixing $\res$, we sample $\bistar'$ according to
  $\Prob{\featuremiss(\sample'_{\istar})=\bistar' \mid \bistar' \neq \mathbf{1}_{\mathbb{R}^d}}$;
  \item we define $\mask'=(\bistar',\res)$.
\end{itemize}
\vspace{-3ex}
First, let us define the following distribution on the $\istar$-th mask under $\sample_{\istar}'$
\begin{align*}
\forall \masksample'_{\istar} \in \{0,1\}^{d},
\nu_{\sample'_{\istar}}(\bistar')
  = \frac{1}{p_{\data' \mid \data}}\Ind{\bistar' \neq \mathbf{1}_{\mathbb{R}^d}}
\Prob{\featuremiss(\sample'_{\istar})=\bistar'}.
\end{align*}
 By definition of $p_{\data' \mid \data}$, the previous sums to $1$ and thus defines a probability distribution.
Then, we introduce
\begin{align*}
 \pi_{2}(\mask, \mask')=
  \pi_{2}((\bistar,\res),(\bistar',\res'))
  = \mu^{(1)}_{\data}((\bistar,\res))\,
  \nu_{\sample'_{\istar}}(\bistar')\,
  \Ind{\res=\res'}.
\end{align*}

We show that $\pi_{2} \in \mathcal{C}(\mu^{(1)}_{\data},\mu^{(1)}_{\data'})$.
First, observe that for all $\mask \in \{0, 1\}^{n \times d}$
\begin{align*}
  \sum_{\mask' \in \{0,1\}^{n\times d}} \pi_{2}(\mask,\mask')
  &= \sum_{(\bistar',\res') \in \{0,1\}^{n \times d}} \pi_{2}((\bistar,\res),(\bistar',\res'))
  \\& = \sum_{(\bistar',\res') \in \{0,1\}^{n \times d}}\mu^{(1)}_{\data}((\bistar,\res))\,
  \nu_{\sample'_{\istar}}(\bistar')\,
  \Ind{\res=\res'}
 \\& = \sum_{\bistar' \in \{0,1\}^d} \mu^{(1)}_{\data}((\bistar,\res))\,\nu_{\sample'_{\istar}}(\bistar') \\
 &= \mu^{(1)}_{\data}((\bistar,\res))\sum_{\bistar' \in \{0,1\}^d}\nu_{\sample'_{\istar}}(\bistar')  \\
  &= \mu^{(1)}_{\data}((\bistar,\res))
  \\&= \mu^{(1)}_{\data}(\mask).
\end{align*}
Secondly, we also have for all $\mask' \in \{0, 1\}^{n \times d}$
\begin{align*}
  \sum_{\mask \in \{0,1\}^{n\times d}} \pi_{2}(\mask,\mask')
  &= \sum_{(\bistar,\res) \in \{0,1\}^{n \times d}} \pi_{2}((\bistar,\res),(\bistar',\res'))
  \\&= \sum_{(\bistar,\res) \in \{0,1\}^{n \times d}}\mu^{(1)}_{\data}((\bistar,\res))\,
  \nu_{\sample'_{\istar}}(\bistar')\,
  \Ind{\res=\res'}.
  \\&= \sum_{\bistar \in \{0,1\}^d} \mu^{(1)}_{\data}((\bistar,\res'))\,\nu_{\sample'_{\istar}}(\bistar') \\
  &= \nu_{\sample'_{\istar}}(\bistar') \sum_{\bistar \in \{0,1\}^d}
  \frac{1}{p_{\data \mid \data'}}\Prob{\featuremiss(\sample_{\istar})=\bistar}\Prob{\mathcal{R}(\data_{\text{res}})=\res'}\Ind{\bistar \neq \mathbf{1}_{\mathbb{R}^d}} \\
  &= \nu_{\sample'_{\istar}}(\bistar')\,\Prob{\mathcal{R}(\data_{\text{res}})=\res'} \\
  &= \nu_{\sample'_{\istar}}(\bistar')\,\Prob{\mathcal{R}(\data'_{\text{res}})=\res'} \quad \text{(because } \data\simeq\data' \text{ then } \data_{\text{res}} = \data'_{\text{res}} \text{)}  \\
 & = \mu^{(1)}_{\data'}((\bistar',\res')) \\
  &= \mu^{(1)}_{\data'}(\mask').
\end{align*}

By construction, on the support of $\pi_{2}$ we have $\res=\res'$.
Since $\data$ and $\data'$ coincide on all samples $i\neq \istar$, it follows that for all $\mask, \mask' \in \{0,1\}^{n \times d},$ 
$\tilde{\data}(\mask)\simeq \tilde{\data}'(\mask')$ iif $\pi_{2}(\mask, \mask') > 0$.

Therefore, substituting \eqref{eq:div_w1_w1p} with $\pi = \pi_{2}$ and by taking the supremum over any events $S \in \mathcal{E}$, we obtain
\begin{align*}
  D_{e^{\epsilon}}(w^{(1)}_{\data} || w^{(1)}_{\data'})
  &\leq \sum_{\mask,\mask' \in \{0,1\}^{n \times d}} \pi_{2}(\mask,\mask')\,D_{e^{\epsilon}}(\algo(\tilde{\data}(\mask))
   || \algo(\tilde{\data}'(\mask')))
\end{align*}
Now, using the fact that $\algo$ is $(\epsilon,\delta)$-DP and the above, we have
\begin{align}
   D_{e^{\epsilon}}(w^{(1)}_{\data} || w^{(1)}_{\data'})
  &\leq \sum_{\mask,\mask' \in \{0,1\}^{n \times d}} \pi_{2}(\mask,\mask')\,D_{e^{\epsilon}}(\algo(\tilde{\data}(\mask))
   || \algo(\tilde{\data}'(\mask')))  \nonumber \\
&\leq \sum_{\mask,\mask' \in \{0,1\}^{n \times d}} \pi_{2}(\mask,\mask') \delta_{\algo}(\epsilon)
    \leq \delta. \label{ineq:b2}
\end{align}

Combining \eqref{ineq:b}, \eqref{ineq:b1} and \eqref{ineq:b2}, we obtain
\begin{align*}
D_{e^{\epsilon'(\data \vert \data')}}(\algo_{\datamiss}(S)\|P) \leq p_{\data \mid \data'} \delta
\end{align*}

To conclude we combine \eqref{eq:key} with the latter bound and prove \eqref{eq:wcast_dp} by denoting $\delta_{\data, \data'} = p_{\data \mid \data'}\delta
+
e^{\epsilon'(\data \vert \data')}|p_{\data \mid \data'} - p_{\data' \mid \data}|$
\begin{align*}
D_{e^{\epsilon'(\data \vert \data')}}(\algo_{\datamiss}(\data)\|\algo_{\datamiss}(\data'))
\le
\delta_{\data\vert\data'}.
\end{align*}

\textbf{II) Proof of \Cref{eq:mnar_dp}}

Let us fix an arbitrary event $S \in \mathcal{E}$.
Similarly to \textbf{I.3)} and \textbf{I.4)}, we use the coupling method. We build a coupling
$\pi_3$ between the mask distributions $\datamiss(\data)$ and $\datamiss(\data')$ such that, on its support,
$\tilde{\data}(\mask) \simeq \tilde{\data}'(\mask')$.

To achieve this, we use the decomposition $\mask=(\bistar,\res)$ introduced in
Definition~\ref{def:res}, where $\bistar$ is the mask of the $\istar$-th sample and
$\res$ is the residual mask, and proceed as follows:
\begin{itemize}
    \item we sample the residual mask $\res$ from the distribution of
    $\mathcal{R}(\data_{\mathrm{res}})$;
    \item we sample two masks $\masksample_{\istar}$ from
    $\featuremiss(\sample_{\istar})$ and $\masksample'_{\istar}$ from
    $\featuremiss(\sample'_{\istar})$;
    \item we set $\mask = (\masksample_{\istar}, \res)$ and
    $\mask' = (\masksample'_{\istar}, \res)$.
\end{itemize}

Formally, for every $\mask=(\bistar,\res)$ and
$\mask'=(\bistar',\res')$ in $\{0,1\}^{n \times d}$, define
\begin{align*}
\pi_{3}\big((\bistar,\res),(\bistar',\res')\big)
&=
\Prob{\mathcal R(\data_{\mathrm{res}})=\res}
\Prob{\featuremiss(\sample_{\istar})=\bistar}
\Prob{\featuremiss(\sample'_{\istar})=\bistar'}
\Ind{\res=\res'}.
\end{align*}

We verify that $\pi_{3}$ is indeed a coupling of the distributions of
$\datamiss(\data)$ and $\datamiss(\data')$. For every
$\mask=(\bistar,\res)$ in $\{0,1\}^{n \times d}$, we have
\begin{align*}
\sum_{\mask' \in \{0,1\}^{n \times d}} \pi_{3}(\mask,\mask')
&=
\sum_{(\bistar',\res') \in \{0,1\}^{n \times d}}
\Prob{\mathcal R(\data_{\mathrm{res}})=\res}
\Prob{\featuremiss(\sample_{\istar})=\bistar}
\Prob{\featuremiss(\sample'_{\istar})=\bistar'}
\Ind{\res=\res'} \\
&=
\Prob{\mathcal R(\data_{\mathrm{res}})=\res}
\Prob{\featuremiss(\sample_{\istar})=\bistar}
\sum_{\bistar' \in \{0,1\}^d}
\Prob{\featuremiss(\sample'_{\istar})=\bistar'} \\
&=
\Prob{\mathcal R(\data_{\mathrm{res}})=\res}
\Prob{\featuremiss(\sample_{\istar})=\bistar} \\
&=
\Prob{\datamiss(\data)=\mask}.
\end{align*}
Similarly, for every $\mask'=(\bistar',\res')$ in $\{0,1\}^{n \times d}$,
\begin{align*}
\sum_{\mask\in \{0,1\}^{n \times d}} \pi_{3}(\mask,\mask')
&=
\sum_{(\bistar,\res) \in \{0,1\}^{n \times d}}
\Prob{\mathcal R(\data_{\mathrm{res}})=\res}
\Prob{\featuremiss(\sample_{\istar})=\bistar}
\Prob{\featuremiss(\sample'_{\istar})=\bistar'}
\Ind{\res=\res'} \\
&=
\Prob{\featuremiss(\sample'_{\istar})=\bistar'}
\Prob{\mathcal R(\data_{\mathrm{res}})=\res'} \\
&=
\Prob{\featuremiss(\sample'_{\istar})=\bistar'}
\Prob{\mathcal R(\data'_{\mathrm{res}})=\res'}
\quad
\text{(since } \data_{\mathrm{res}}=\data'_{\mathrm{res}} \text{)} \\
&=
\Prob{\datamiss(\data')=\mask'}.
\end{align*}

Since, by construction, for any masks $\mask,\mask'\in\{0,1\}^{n \times d}$
such that $\pi_{3}(\mask,\mask')>0$, we have
$\tilde{\data}(\mask)\simeq \tilde{\data}'(\mask')$, and since $\algo$ is
$(\epsilon,\delta)$-DP on $\missdatadomain^n$, we have
\[
D_{e^{\epsilon}}\!\left(
\algo(\tilde{\data}(\mask))
\middle\|
\algo(\tilde{\data}'(\mask'))
\right)
\leq \delta
\]
for all such $(\mask,\mask')$. Therefore, we prove \Cref{eq:mnar_dp}:
\begin{align*}
D_{e^{\epsilon}}\!\left(
\algo_{\datamiss}(\data)
\middle\|
\algo_{\datamiss}(\data')
\right)
&\le
\sum_{\mask,\mask' \in \{0,1\}^{n \times d}}
\pi_3(\mask,\mask')
D_{e^{\epsilon}}\!\left(
\algo(\tilde{\data}(\mask))
\middle\|
\algo(\tilde{\data}'(\mask'))
\right) \\
&\le
\sum_{\mask,\mask' \in \{0,1\}^{n \times d}}
\pi_3(\mask,\mask')\delta \\
&= \delta.
\end{align*}
\textbf{III) Conclusion}

Let us introduce $\epsilon(\lambda, \data) = \log(\lambda e^{\epsilon} + (1-\lambda)e^{\epsilon'(\data \vert \data')})$. For any event $S \in \mathcal{E}$, we have
\begin{align*}
\Prob{\algo_{\datamiss}(\data) \in S} - &e^{\epsilon(\lambda, \data)}\Prob{\algo_{\datamiss}(\data') \in S} \\&= \Prob{\algo_{\datamiss}(\data) \in S} - \lambda e^{\epsilon}\Prob{\algo_{\datamiss}(\data') \in S} - (1-\lambda)e^{\epsilon'(\data \vert \data')}\Prob{\algo_{\datamiss}(\data') \in S} 
\\& = \lambda \Prob{\algo_{\datamiss}(\data) \in S} - \lambda e^{\epsilon}\Prob{\algo_{\datamiss}(\data') \in S} +(1-\lambda)\Prob{\algo_{\datamiss}(\data) \in S} - (1-\lambda)e^{\epsilon'(\data \vert \data')}\Prob{\algo_{\datamiss}(\data') \in S}\\&= \lambda\left( \Prob{\algo_{\datamiss}(\data) \in S}-e^{\epsilon}\Prob{\algo_{\datamiss}(\data') \in S}\right) + (1-\lambda)\left(\Prob{\algo_{\datamiss}(\data) \in S} -  e^{\epsilon'(\data \vert \data')}\Prob{\algo_{\datamiss}(\data') \in S}\right).
\end{align*}
By taking the supremum over all $S \in \mathcal{E}$, we finally obtain
\begin{align*}
\Divmiss{e^{\epsilon(\lambda, \data)}} &\leq \lambda \Divmiss{e^{\epsilon}} + (1-\lambda)\Divmiss{e^{\epsilon'(\data \vert \data')}} 
\\ &\le \lambda\delta + (1-\lambda)\delta_{\data\vert\data'}.
\end{align*}

\end{proof}

\subsection{Proof of Lemma~\ref{lemma:pstar2}}\label{app:lemmapstar}
\vspace{3ex}
\begin{mdframed}
\begin{restatable}{lemma}{LemmaPstar}\label{lemma:pstar2}
If the missing data process $\datamiss$ is MAR, there exists a constant $\qstar\in[0,1]$ such that for all neighboring
datasets $\data\simeq\data'$ in $\datadomain^n$, we have
\begin{align}
\qstar = \Prob{\datamiss(\data)\in \Hstar(\data,\data')}
=
\Prob{\datamiss(\data')\in \Hstar(\data,\data')} \label{eq:pstardef}.
\end{align}
\end{restatable}
\end{mdframed}
\begin{proof}

Let $\featuremiss$ be the missing feature process that induces $\datamiss$. By definition, if $\datamiss$ is MAR, then $\featuremiss$ is also MAR.
Consequently, the probability of a mask may depend only on the observed features. Furthermore, since there are no observed features associated with the mask $\mathbf{1}_{\mathbb{R}^d}$, its probability under the missing feature process $\mathcal{F}$ does not depend on the data. Therefore there exists a constant $q_{\ast}$ such that for any sample $\sample \in \datadomain$
\begin{align}
  \Prob{\featuremiss(\sample) = \mathbf{1}_{\mathbb{R}^d}} = q_{\ast} \label{eq:qstardef}.
\end{align}

We denote by $\Hstar^c$, the complement of $\Hstar$ in $\{0,1\}^{n }$. Then, by \eqref{def:missingmech}, we can write, for any $\data \in \datadomain^n$,
\begin{align*}
\Prob{\datamiss(\data)\in \Hstar^c(\data,\data')}
&= \sum_{\mask \in \Hstar^c(\data, \data')} \Prob{\datamiss(\data)=\mask}\\
&= \sum_{\mask \in \Hstar^c(\data, \data')}\left( \prod_{i=1}^n \Prob{\featuremiss(\sample_i)=\mask_i}\right).
\end{align*}

Recalling that $\mask \in \Hstar^c(\data, \data')$ is equivalent to $\masksample_{\istar}=\mathbf{1}_{\mathbb{R}^d}$, we have 
\begin{align*}
  \Prob{\datamiss(\data)\in \Hstar^c(\data,\data')}
&= \sum_{(\masksample_i)_{i\neq \istar}: \masksample_i \in \{0,1\}^d}
\left(
\Prob{\featuremiss(\sample_{\istar})=\mathbf{1}_{\mathbb{R}^d}}\cdot
\prod_{i\neq \istar}^n\Prob{\featuremiss(\sample_i)=\masksample_i}
\right) \\
\end{align*}
Since the missing data process is independent per sample, for any data mask $\mask$ in $\{0,1\}^{n\times d}$, we use the decomposition with the residual mask $\res$, the residual dataset $\data_\text{res}$ and the residual missing data process $\mathcal{R}$ introduced in Definition~\ref{def:res}. Doing so, we get 
\begin{align*}
\Prob{\datamiss(\data)\in \Hstar^c(\data,\data')} &= \sum_{\res \in \{0,1\}^{(n-1)\times d}}\Prob{\featuremiss(\sample_{\istar})=\mathbf{1}_{\mathbb{R}^d}} \Prob{\mathcal{R}(\data_{\text{res}}) = \res} \\
&= \Prob{\featuremiss(\sample_{\istar})=\mathbf{1}_{\mathbb{R}^d}}\sum_{\res \in \{0,1\}^{(n-1)\times d}} \Prob{\mathcal{R}(\data_{\text{res}}) = \res} \\
&= \Prob{\featuremiss(\sample_{\istar})=\mathbf{1}_{\mathbb{R}^d}} \quad \text{(since } \mathcal{R}(\data_{\text{res}}) \text{ is a random variable with values in } \{0,1\}^{(n-1)\times d}\text{)} \\
&= q_{\ast} \quad \text{(by definition of } q_{\star} \text{ in \cref{eq:qstardef})}.
\end{align*}

Similarly, we can show that
\begin{align*}
\Prob{\datamiss(\data')\in \Hstar^c(\data,\data')}
&= \Prob{\featuremiss(\sample'_{\istar})=\mathbf{1}_{\mathbb{R}^d}} =q_{\ast}.
\end{align*}

Therefore, we have
\begin{align*}
  \Prob{\datamiss(\data)\in \Hstar^c(\data,\data')} = \Prob{\datamiss(\data')\in \Hstar^c(\data,\data')} = q_{\ast}.
\end{align*}

Finally, since $\Hstar(\data,\data')$ is the complement of $\Hstar^c(\data,\data')$, we conclude by setting $\qstar = 1 - q_{\ast}$.
\end{proof}

\subsection{Privacy amplification for MAR (and MCAR)}
\vspace{3ex}

\begin{mdframed}
\begin{restatable}{corollary}{AmplificationTheorem}\label{thm:amplif} Let $\epsilon>0$, and $\delta \in [0,1]$. Let $\algo$ be an $(\epsilon,\delta)$-DP randomized algorithm. Let $\datamiss$ be a MAR missing data process. Let $\Miss$ be the randomized algorithm defined as per~(\ref{def:A_D}) with $\algo$ and $\datamiss$. Then, we have
\begin{equation*}
  \delta_{\Miss}(\epsilon')
  \;\leq\; p_{\ast}\delta,
\end{equation*}
where $\epsilon' = \ln\!\big(1+p_{\ast}(e^{\epsilon}-1)\big).$ In particular, the same result holds if the missing data process $\datamiss$ is MCAR. 
\end{restatable}
\end{mdframed}

\begin{proof}
Since $\datamiss$ is MAR, by \Cref{lemma:pstar2}, there exists a constant $p_\ast\in[0,1]$ such that for any pair of neighboring datasets $\data \simeq \data'$ in $\datadomain$
\begin{align*}
p_\ast
&=
\Prob{\datamiss(\data)\in\Hstar(\data,\data')}
=
\Prob{\datamiss(\data')\in\Hstar(\data,\data')}.
\end{align*}
Let $\data\simeq\data'$ be two neighboring datasets in $\datadomain^n$.

Applying Theorem~\ref{thm:MNAR} with $\lambda=0$, we obtain
\begin{align*}
D_{e^{\epsilon'}}\!\left(
\Miss(\data)\,\middle\|\,\Miss(\data')
\right)
\leq
p_\ast\delta,
\end{align*}
where
\begin{align*}
\epsilon'
=
\log\!\big(1+p_\ast(e^\epsilon-1)\big).
\end{align*}

Since this bound holds for every neighboring pair $\data\simeq\data'$, taking the supremum over all such pairs yields
\begin{align*}
\delta_{\Miss}(\epsilon')
&=
\sup_{\data\simeq\data'}
D_{e^{\epsilon'}}\!\left(
\Miss(\data)\,\middle\|\,\Miss(\data')
\right)
\leq
p_\ast\delta.
\end{align*}

Finally, since every MCAR missing data process is also MAR, the same result immediately holds in the MCAR case.
\end{proof}

\subsection{No amplification in the $\qstar=1$ regime}\label{app:proptight}
\vspace{3ex}

\begin{mdframed}
\begin{restatable}{proposition}{PropositionTight}~\label{prop:tight}
Let $\datadomain \subset [-1,1]^d$. There exists a non trivial $(\epsilon, \delta)$-DP randomized algorithm $\algo^0$, a non trivial missing data process
$\datamiss^0$ satisfying the hypotheses of \Cref{thm:amplif} such that $\qstar =1$ and:
\begin{align*}
  \delta_{\algo^0_{\mathcal{D}^0}}(\epsilon)
  =
  \delta_{\algo^0}(\epsilon).
\end{align*}
\end{restatable}
\end{mdframed}

\begin{proof}
Fix $j_0\in[d]$.
We explicitly build a missing data process and a Gaussian mechanism
satisfying the claim.

\textbf{Missing data process.}

Let us define a MAR missing data process $\mathcal D^0$ such that the coordinate
$(i,j_0)$ is always observed for every $i\in[n]$, that is for all
$\data\in\datadomain^n$,
\begin{align*}
\Prob{
\mathcal D^{0}(\data)
\in
\{
\mask\in\{0,1\}^{n\times d}
\mid
\mask_i^{(j_0)}=0
\ \forall i\in[n]
\}
}
=
1.
\end{align*}

As a consequence, for any neighboring datasets $\data\simeq\data'$ in $\missdatadomain^n$,
the differing record is always partially observed and therefore $\qstar=1$.

\textbf{Query and Gaussian mechanism.}

Define the real-valued query $\query:\mathcal Z^n_{\mathrm{miss}}\to\mathbb R$ by
\begin{align*}
\query(\tilde\data)
=
\sum_{i=1}^n I(\tilde\data)_i^{(j_0)}.
\end{align*}
where  $I : \missdatadomain^n \rightarrow \datadomain^n$ is a data imputation procedure that replaces all missing values by $0$ and leaves all other values unchanged and $I(\tilde{z})^{(j_0)}_i$ is the $j_0$-th feature of the row $i$ of the imputed dataset $I(\tilde{z})$.

Since $\mask_i^{(j_0)}=0$ almost surely under $\mathcal D^0$, we have
\begin{align*}
\query(\tilde\data(\mask))
=
\sum_{i=1}^n \data_i^{(j_0)}
\end{align*}
for all $\mask$ in the support of $\mathcal D^0(\data)$.

Since $\datadomain \subset [-1,1]^d$, we have 
\begin{align*}
\max_{i\in[n]}\|\sample_i\|_\infty
\le
1.
\end{align*}

It follows that for any two neighboring incomplete datasets
$\tilde\data\simeq\tilde\data'$ in $\missdatadomain^n$,
\begin{align*}
|\query(\tilde\data)-\query(\tilde\data')|
\le
2 ,
\end{align*}
and therefore the $\ell_2$-sensitivity of $\query$ satisfies
\begin{align*}
\querysens{2}
\le
2 .
\end{align*}

Choosing
\begin{align*}
\sigma
=
\frac{c\,\querysens{2}}{\epsilon}
\qquad
\text{with}
\qquad
c>\sqrt{2\ln(1.25/\delta)},
\end{align*}
the Gaussian mechanism
\begin{align*}
\algo^0(\tilde\data)
=
\query(\tilde\data)
+
\mathcal N(0,\sigma^2)
\end{align*}
is $(\epsilon,\delta)$-differentially private.

\paragraph{Equality of divergences.}

Let $\algo^0_{\datamiss^0}$ denote the mechanism incorporating $\algo^0$ and $\datamiss^0$.
For any dataset $\data\in\datadomain^n$ and any event $S\in\mathcal B(\mathbb R)$,
\begin{align*}
\Prob{\algo^0_{\datamiss^0}(\data)\in S}
&=
\sum_{\mask\in\{0,1\}^{n\times d}}
\Prob{\mathcal D^0(\data)=\mask}
\Prob{\algo^0(\tilde\data(\mask))\in S}
\\
&=
\Prob{\algo^0(\data)\in S},
\end{align*}
since
\begin{align*}
\query(\tilde\data(\mask))
=
\sum_{i=1}^n \data_i^{(j_0)}
\end{align*}
almost surely.

Therefore, for any neighboring datasets $\data\simeq\data'$,
\begin{align*}
D_{e^\epsilon}
\!\left(
\algo^0_{\datamiss^0}(\data)
\;\|\;
\algo^0_{\datamiss^0}(\data')
\right)
=
D_{e^\epsilon}
\!\left(
\algo^0(\data)
\;\|\;
\algo^0(\data')
\right).
\end{align*}

Taking the supremum over all neighboring pairs yields
\begin{align*}
\delta_{\algo^0_{\datamiss^0}}(\epsilon)
=
\delta_{\algo^0}(\epsilon).
\end{align*}
\end{proof}

\section{Building FWL queries}
\label{app:FWL}
In this section, we present several results that help practitioners build DP pipelines satisfying the FWL condition.

\subsection{Norm compatibility}

We first show the following norm compatibility for FWL queries.

\begin{mdframed}
\begin{restatable}{lemma}{PropositionFWLnorm}\label{prop:l1l2} 
If a query $\query$ is FWL with respect to  the $\|.\|_1$ norm, then $\query$ is FWL with respect to  any $\|.\|_p$ norm, for any $p>1$, with the same constants.
\end{restatable} 
\end{mdframed}

\begin{proof}
Fix any $p>1$ and two neighboring datasets $\datamask \simeq \datamask'$ in
$\missdatadomain^n$.
Since $\query$ is FWL with respect to \ $\|\cdot\|_1$, we have
\begin{align*}
\|\query(\datamask)-\query(\datamask')\|_1
\le
\sum_{j=1}^d L_j\,\big|I(\samplemask)^{(j)}_{\istar} -I(\samplemask')^{(j)}_{\istar}\big|.
\end{align*}
Besides, for any vector $\mathbf{v} \in\mathbb{R}^k$ and any $p\ge 1$ we have
\begin{align*}
  \|\mathbf{v}\|_p \le \|\mathbf{v}\|_1.
\end{align*}
Applying the previous with
$\mathbf{v}=\query(\datamask)-\query(\datamask')$ for any neighboring datasets $\data, \data'$ in $\missdatadomain^n$ allows us to conclude the proof.
\end{proof}

\subsection{Canonical FWL queries}

We now present several basic classes of queries that satisfy the FWL property.
These examples cover many standard DP workloads: \textbf{(L)} linear dataset-level queries, \textbf{(S)} queries with bounded $\ell_1$ dataset sensitivity,
and \textbf{(Q)} quadratic forms on bounded domains.
\vspace{3ex}
\begin{mdframed}
\begin{restatable}{proposition}{PropositionFWLNativeA}\label{prop:typeA-clean}

Let $\query:\missdatadomain^n\to\mathbb{R}^k$ be a dataset-level query 

\textbf{(L) Linear dataset-level queries:} for all $\datamask \in \missdatadomain^n$ the linear query $\query$ is defined by
\begin{align*}
\query(\datamask)=\sum_{i=1}^n B_i I(\samplemask)_i, \quad \text{with }
 B_i\in\mathbb{R}^{k\times d}.
\end{align*}
Then $\query$ is FWL under $\|\cdot\|_1$ with constants defined for all $j \in [d]$ as
\begin{align*}
L_j=\max_{i\in[n]} \|B_i e_j\|_1,
\end{align*}
where $e_j$ denotes the $j$-th canonical basis vector of $\mathbb{R}^d$.

\textbf{(S) Dataset-level $\ell_1$-Lipschitz queries:} for all $\datamask, \datamask'$ in $\missdatadomain^n$, we define the $L$-Lipschitz query by
\begin{align*}
  \|\query(\datamask)-\query(\datamask')\|_1 \le L\, \sum_{i=1}^n \|I(\samplemask)_i-I(\samplemask')_i\|_1.
\end{align*}

Then, $\query$ is FWL under $\|\cdot\|_1$ with constants $L_j = L$,  for all $j\in[d]$.

\paragraph{(Q) Quadratic queries on a bounded domain.}
Assume $\missdatadomain\subset([-B,B]\cup\{\texttt{NA}\})^d$. We define the quadratic query for all $\datamask$ in $\missdatadomain$ by
\begin{align*}
\query(\datamask)=\frac1n\sum_{i=1}^n I(\samplemask)_iI(\samplemask)_i^\top \in \mathbb{R}^{d\times d}.
\end{align*}
Then, the vectorized query $\bar\query$ is FWL under $\|\cdot\|_1$ with constants defined for all $j \in [d]$ by
\begin{align*}
  L_j=\frac{2Bd}{n}.
\end{align*}
\end{restatable}
\end{mdframed}

\begin{proof}
We prove each item separately. Let us fix two neighboring datasets $\datamask\simeq\datamask'$ in $\missdatadomain^n$ differing in a unique record $\istar$.

\paragraph{Proof of (L).}
We have
\begin{align*}
\query(\datamask)-\query(\datamask')
&= \sum_{i=1}^n B_i I(\samplemask)_i - \sum_{i=1}^n B_i I(\samplemask')_i
= B_{\istar}I(\samplemask)_{\istar}-B_{\istar}I(\samplemask')_{\istar}
= B_{\istar}(I(\samplemask)_{\istar}-I(\samplemask')_{\istar}).
\end{align*}
Next, we expand the difference $I(\samplemask)_{\istar}-I(\samplemask')_{\istar}$ in the canonical basis of $\mathbb{R}^d$
\begin{align*}
I(\samplemask)_{\istar}-I(\samplemask')_{\istar}
= \sum_{j=1}^d \big(I(\samplemask)_{\istar}^{(j)}-I(\samplemask')_{\istar}^{(j)}\big)\,e_j.
\end{align*}
By linearity of $B_{\istar}$, it yields
\begin{align*}
B_{\istar}(I(\samplemask)_{\istar}-I(\samplemask')_{\istar})
= \sum_{j=1}^d \big(I(\samplemask)_{\istar}^{(j)}-I(\samplemask')_{\istar}^{(j)}\big)\,B_{\istar}e_j.
\end{align*}

Taking $\|.\|_1$ norm and applying the triangle inequality we obtain
\begin{align*}
\|\query(\datamask)-\query(\datamask')\|_1
&= \big\|B_{\istar}(I(\samplemask)_{\istar}-I(\samplemask_{\istar})')\big\|_1 \\
&= \Big\|\sum_{j=1}^d \big(I(\samplemask)_{\istar}^{(j)}-I(\samplemask')_{\istar}^{(j)}\big)\,B_{\istar}e_j\Big\|_1 \\
&\le \sum_{j=1}^d \big|I(\samplemask)_{\istar}^{(j)}-I(\samplemask')_{\istar}^{(j)}\big|\,\|B_{\istar}e_j\|_1.
\end{align*}
We define $L_j \eqdef \max_{i\in[n]}\|B_i e_j\|_1$. Then, by definition, $\|B_{\istar}e_j\|_1\le L_j$ for all $j$, hence,
\begin{align*}
\|\query(\datamask)-\query(\datamask')\|_1
\le \sum_{j=1}^d L_j\,\big|I(\samplemask)_{\istar}^{(j)}-I(\samplemask')_{\istar}^{(j)}\big|.
\end{align*}

\paragraph{Proof of (S).}
By assumption, we have
\begin{align*}
\|\query(\datamask)-\query(\datamask')\|_1 \le L\sum_{i=1}^n \|I(\samplemask)_i-I(\samplemask')_i\|_1.
\end{align*}
Since, for all $i\neq \istar$, $I(\samplemask)_i=I(\samplemask')_i$, we obtain
\begin{align*}
\sum_{i=1}^n \|I(\samplemask)_i-I(\samplemask')_i\|_1 = \|I(\samplemask)_{\istar}-I(\samplemask')_{\istar}\|_1.
\end{align*}
Expanding the $\ell_1$ norm coordinate-wise yields to
\begin{align*}
\|I(\samplemask)_{\istar}-I(\samplemask')_{\istar}\|_1
= \sum_{j=1}^d |I(\samplemask)_{\istar}^{(j)}-I(\samplemask')_{\istar}^{(j)}|.
\end{align*}
Therefore,
\begin{align*}
\|\query(\datamask)-\query(\datamask')\|_1
\le\sum_{j=1}^d L\,|I(\samplemask)_{\istar}^{(j)}-I(\samplemask')_{\istar}^{(j)}|.
\end{align*}
This is the FWL inequality under $\|\cdot\|_1$ with constants $L_j=L$.

\paragraph{Proof of (Q).}
We have
\begin{align*}
\query(\datamask)-\query(\datamask')
=\frac{1}{n}\Big(I(\samplemask)_{\istar}I(\samplemask)_{\istar}^\top
- I(\samplemask')_{\istar}I(\samplemask')_{\istar}^\top\Big).
\end{align*}
Using the identity
\begin{align*}
I(\samplemask)_{\istar}I(\samplemask)_{\istar}^\top
- I(\samplemask')_{\istar}I(\samplemask')_{\istar}^\top
=
(I(\samplemask)_{\istar}-I(\samplemask')_{\istar})I(\samplemask)_{\istar}^\top
+
I(\samplemask)_{\istar}'(I(\samplemask)_{\istar}-I(\samplemask')_{\istar})^\top,
\end{align*}
and the triangle inequality, we obtain
\begin{align*}
\big\|\bar\query(\datamask)-\bar\query(\datamask')\big\|_1
&=\frac{1}{n}\Big\|\mathrm{vec}\big(I(\samplemask)_{\istar}I(\samplemask)_{\istar}^\top
- I(\samplemask')_{\istar}I(\samplemask')_{\istar}^\top\big)\Big\|_1\\
&\le \frac{1}{n}\Big(
\big\|\mathrm{vec}\big((\samplemask_{\istar}-\samplemask_{\istar}')\samplemask_{\istar}^\top\big)\big\|_1
+
\big\|\mathrm{vec}\big(I(\samplemask)_{\istar}'(I(\samplemask)_{\istar}-I(\samplemask')_{\istar})^\top\big)\big\|_1
\Big).
\end{align*}
For any vectors $\mathbf{u},\mathbf{v}\in\mathbb{R}^d$, the rank-one matrix $\mathbf{u}\mathbf{v}^\top$ satisfies
\begin{align*}
\big\|\mathrm{vec}(\mathbf{u}\mathbf{v}^\top)\big\|_1
=\sum_{k=1}^d\sum_{\ell=1}^d |u_k v_\ell|
=\Big(\sum_{k=1}^d |u_k|\Big)\Big(\sum_{\ell=1}^d |v_\ell|\Big)
=\|\mathbf{u}\|_1\|\mathbf{v}\|_1.
\end{align*}
Applying this identity yields to
\begin{align*}
\big\|\bar\query(\datamask)-\bar\query(\datamask')\big\|_1
&\le \frac{1}{n}\Big(
\|I(\samplemask)_{\istar}-I(\samplemask')_{\istar}\|_1\,\|I(\samplemask)_{\istar}\|_1
+
\|I(\samplemask')_{\istar}\|_1\,\|I(\samplemask)_{\istar}-I(\samplemask')_{\istar}\|_1
\Big)\\
&=\frac{\|I(\samplemask)_{\istar}\|_1+\|I(\samplemask')_{\istar}\|_1}{n}\,\|I(\samplemask)_{\istar}-I(\samplemask')_{\istar}\|_1.
\end{align*}
Since $\samplemask_{\istar},\samplemask_{\istar}'\in[-B,B]^d$, we have $\|I(\samplemask)_{\istar}\|_1,\|I(\samplemask')_{\istar}\|_1\le dB$, and therefore
\begin{align*}
\big\|\bar\query(\datamask)-\bar\query(\datamask')\big\|_1
&\le \frac{2dB}{n}\,\|I(\samplemask)_{\istar}-I(\samplemask')_{\istar}'|_1
= \frac{2dB}{n}\sum_{j=1}^d
\big|I(\samplemask)_{\istar}^{(j)}-I(\samplemask')_{\istar}^{(j)}\big|.
\end{align*}
Thus $\bar\query$ is FWL under $\|\cdot\|_1$ with constants $L_j=\frac{2dB}{n}$ for all $j\in[d]$.

\end{proof}

\subsection{Closure properties}

Beyond basic queries, the FWL property is preserved under the following operations.

\begin{mdframed}
  \begin{proposition}\label{prop:FWL-closure}
Let $\query:\missdatadomain^n\to\mathbb{R}^k$ be an FWL query with constants $(L_j)_{j\in[d]}$ under the norm $\|.\|$.
Let $\psi:\mathbb{R}^k\to\mathbb{R}^{k'}$ be $\Lambda$-Lipschitz with $k' \geq 1$ with respect to  the same norm $\|.\|$.

Then $\psi\circ \query$ is FWL with constants $(\Lambda L_j)_{j\in[d]}$ under $\|.\|$.
\end{proposition}
\end{mdframed}

\begin{proof}
Fix two neighboring datasets $\datamask \simeq \datamask'$ in $\missdatadomain^n$.

Since $\psi$ is $\Lambda$-Lipschitz
(with respect to the same norm, we have
\begin{align*}
\|\psi(\query(\datamask))-\psi(\query(\datamask'))\|
\le \Lambda \,\|\query(\datamask)-\query(\datamask')\|.
\end{align*}
Because $\query$ is FWL with constants $(L_j)_{j\in[d]}$, it follows that
\begin{align*}
\|\query(\datamask)-\query(\datamask')\|
\le \sum_{j=1}^d L_j\,\big|I(\tilde{\sample})^{(j)}_{\istar} - I(\tilde{\sample'})^{(j)}_{\istar} \big|.
\end{align*}
Combining the two inequalities yields to
\begin{align*}
\|(\psi\circ\query)(\datamask)-(\psi\circ\query)(\datamask')\|
\le \sum_{j=1}^d \Lambda L_j\big|I(\samplemask)^{(j)}_{\istar} - I(\samplemask')^{(j)}_{\istar}\big|.
\end{align*}
\end{proof}
\begin{mdframed}
\begin{proposition}\label{prop:FWL-lincomb}
Let $\|\cdot\|$ be any norm on $\mathbb{R}^k$. For $l\in[L]$, let
$f_{I,l}:\missdatadomain^n\to\mathbb{R}^k$ be FWL under\ $\|\cdot\|$ with
constants $(L^{(r)}_{j})_{j\in[d]}$. Let $a_1,\dots,a_L\in\mathbb{R}$. We define for any $\datamask \in \missdatadomain^n$
\begin{align*}
g_{\text{LC}}(\datamask) = \sum_{l=1}^L a_l\,f_{I,l}(\datamask).
\end{align*}
Then $g_{\text{LC}}$ is FWL under \ $\|\cdot\|$ with constants $(L_j)_{j\in[d]}$ given by, for all $ j\in[d]$,
\begin{align*}
L_j = \sum_{l=1}^L |a_l|\,L^{(l)}_{j}.
\end{align*}
\end{proposition}
\end{mdframed}

\begin{proof}
Fix neighboring datasets $\tilde{\data}\simeq\tilde{\data}' \in \missdatadomain^n$. By the triangle inequality and homogeneity of the norm, we obtain
\begin{align*}
  \|g_{\text{LC}}(\datamask)-g_{\text{LC}}(\datamask')\|
=
\Big\|\sum_{l=1}^L a_l\big(f_{I,l}(\datamask)-f_{I,l}(\datamask')\big)\Big\|
\le
\sum_{r=l}^L |a_l|\,\|f_{I,l}(\datamask)-f_{I,l}(\datamask')\|.
\end{align*}

Since each $f_{I,l}$ is FWL, we have
\begin{align*}
\|f_{I,l}(\datamask)-f_{I,l}(\datamask')\|
\le
\sum_{j=1}^d L^{(l)}_{j}\,\big|\ I(\samplemask)^{(j)}_{\istar} - I(\samplemask')^{(j)}_{\istar}\big|.
\end{align*}
Combining the previous inequalities yields 
\begin{align*}
\|g_{\text{LC}}(\datamask)-g_{\text{LC}}(\datamask')\|
\le
\sum_{j=1}^d\Big(\sum_{l=1}^L |a_l|\,L^{(l)}_{j}\Big)
\big|I(\samplemask)^{(j)}_{\istar} - I(\samplemask')^{(j)}_{\istar} \big|.
\end{align*}
\end{proof}

\section{Proofs of Section~\ref{sec:fwl}}

\subsection{Auxiliary Lemmas}\label{app:plemma:la}
\vspace{3ex}
\begin{mdframed}
\begin{restatable}{lemma}{LemmaLap}\label{lemma:la}
Let $\query : \missdatadomain^n \to \mathbb{R}^k$ be a query with $\ell_1$ sensitivity $\querysens{1} \leq \Delta$.
Then $\lap$ is $\big((\querysens{1} / \Delta)\epsilon, 0\big)$-DP.
\end{restatable}
\end{mdframed}

\begin{proof}
Fix neighboring datasets $\tilde{\data} \simeq \tilde{\data}'$ in $\missdatadomain^n$.
For $t\in\mathbb{R}$, the Laplace density with scale $b$ is defined by
\begin{align*}
  \ell_b(t) = \frac{1}{2b}\exp\!\left(-\frac{|t|}{b}\right).
\end{align*}

By independence, the density of $\lap(\tilde{\data})$ at $\mathbf{t}=(t_1,\dots,t_k)$ is denoted by
\begin{align*}
q_{\tilde{\data}}(\mathbf{t})
=
\prod_{i=1}^k \ell_b\!\bigl(t_i-\query(\tilde{\data})_i\bigr)
=
\left(\frac{1}{2b}\right)^k
\exp\!\left(
-\frac{1}{b}\sum_{i=1}^k \bigl|t_i-\query(\tilde{\data})_i\bigr|
\right).
\end{align*}

Similarly, the density of $\lap(\tilde{\data}')$ is
\begin{align*}
q_{\tilde{\data}'}(\mathbf{t})
=
\left(\frac{1}{2b}\right)^k
\exp\!\left(
-\frac{1}{b}\sum_{i=1}^k \bigl|t_i-\query(\tilde{\data}')_i\bigr|
\right).
\end{align*}
Therefore,
\begin{align*}
\frac{q_{\data}(\mathbf{t})}{q_{\data'}(\mathbf{t})}
&=
\exp\!\left(
\frac{1}{b}\sum_{i=1}^k
\Bigl(\bigl|t_i-\query(\tilde{\data}')_i\bigr|-\bigl|t_i-\query(\tilde{\data})_i\bigr|\Bigr)
\right).
\end{align*}
Using the triangle inequality, we obtain for all $i \in [k]$
\begin{align*}
\bigl|t_i-\query(\tilde{\data}')\bigr|-\bigl|t_i-\query(\tilde{\data})_i\bigr|
\le
\bigl|\bigl(t_i-\query(\tilde{\data}')_i\bigr)-\bigl(t_i-\query(\tilde{\data})_i\bigr)\bigr|
=
\bigl|\query(\tilde{\data})_i-\query(\tilde{\data}')_i\bigr|.  
\end{align*}

Summing over $i\in[k]$ and using the definition of $\querysens{1}$ yields
\begin{align*}
\sum_{i=1}^k
\Bigl(\bigl|t_i-\query(\tilde{\data}')_i\bigr|-\bigl|t_i-\query(\tilde{\data})_i\bigr|\Bigr)
\le
\sum_{i=1}^k |\query(\tilde{\data})_i-\query(\tilde{\data}')_i|
=
\|\query(\tilde{\data})-\query(\tilde{\data}')\|_1
\le
\querysens{1}.  
\end{align*}

Hence, for all $\mathbf{t}\in\mathbb{R}^k$, we have
\begin{align*}
\frac{q_{\tilde{\data}}(\mathbf{t})}{q_{\tilde{\data}'}(\mathbf{t})}
\le
\exp\!\left(\frac{\querysens{1}}{b}\right).
\end{align*}

Let \(S\subseteq\mathbb{R}^k\) be any measurable set in $\mathcal{E}$. Integrating over \(S\) yields to
\begin{align*}
\Prob{\lap(\tilde{\data})\in S}
=
\int_S q_{\tilde{\data}}(\mathbf{t})\,d\mathbf{t}
\le
\exp\!\left(\frac{\querysens{1}}{b}\right)\int_S q_{\tilde{\data}'}(\mathbf{t})\,d\mathbf{t}
=
\exp\!\left(\frac{\querysens{1}}{b}\right)\Prob{\lap(\tilde{\data}')\in S},
\end{align*}
which proves that $\lap$ is $(\querysens{1}/b,0)$-differentially private.
\end{proof}

\begin{mdframed}
\begin{restatable}{lemma}{LemmaGauss}\label{lemma:gauss}
  Let $\query : \missdatadomain^n \to \mathbb{R}^k$ be a query with $\ell_2$ sensitivity $\querysens{2} \leq \Delta$.
Then $\gauss$ is $\big((\querysens{2} / \Delta)\epsilon, \delta\big)$-DP.
\end{restatable}
\end{mdframed}

\begin{proof}
  Let us set
  \begin{align}
    \epsilon' = \frac{\querysens{2}}{\Delta}\epsilon.
  \end{align}
  Then, we have
  \begin{align}
    \sigma = \frac{c \Delta}{\epsilon} = \frac{c \querysens{2}}{\epsilon'}
  \end{align}
  Applying Theorem A.1 \cite{dwork2006differential}, the Gaussian mechanism with sensitivity $\querysens{2}$ is $(\epsilon', \delta)$-DP. Hence $\gauss$ is $\big((\querysens{2} / \Delta)\epsilon, \delta\big)$-DP.
\end{proof}

\subsection{Proof of \Cref{lemma:sensmiss}}\label{app:sensmiss}
\vspace{3ex}
\begin{mdframed}
  \Propsensitivity*
\end{mdframed}

\begin{proof}
Let us fix two neighboring datasets $\data\simeq\data'$ and let $i^\star$ be the index of the differing
record. We also fix two mask $\mask, \mask'$ respectively in $\textrm{supp}(\datamiss(\data)$ and 
$\textrm{supp}(\datamiss(\data')$ such that 
$\tilde{\data}(\mask)\simeq \tilde{\data}'(\mask')$.

By \Cref{ass:proportion_observed_entries},
\begin{align}
    |\operatorname{obs}(\masksample_{\istar})|\le \lfloor \rho d\rfloor,
    \qquad
    |\operatorname{obs}(\masksample'_{\istar})|\le \lfloor \rho d\rfloor. \label{eq:cardistar}
\end{align}

Since $\datadomain\subset[-1,1]^d$, for any $j \in [d]$, we have
\begin{align*}
\begin{aligned}
\left|
    I(\tilde \sample)_{\istar}^{(j)}(\mask)
    -
    I(\tilde \sample')_{\istar}^{(j)}(\mask')
\right|
&\le
    |\sample_{\istar}^{(j)}|\Ind{\{j\in \operatorname{obs}(\masksample_{\istar})\}}
    +
    |\sample_{\istar}^{\prime(j)}|\Ind{\{j\in \operatorname{obs}(\masksample'_{\istar})\}}  \\
&\le
    \Ind{\{j\in \operatorname{obs}(\masksample_{\istar})\}}
    +
    \Ind{\{j\in \operatorname{obs}(\masksample'_{\istar})\}} .
\end{aligned}
\end{align*}
Since $\query$ is FWL,
\begin{align*}
\begin{aligned}
\left\|
    \query(\tilde{\data}(\mask))
    -
    \query(\tilde{\data}'(\mask'))
\right\|_p
&\le
\sum_{j=1}^d
L_j
\left|
    I(\tilde \sample)_{\istar}^{(j)}(\mask)
    -
    I(\tilde \sample')_{\istar}^{(j)}(\mask')
\right|  \\
&\le
\sum_{j\in \operatorname{obs}(\masksample_{\istar})}L_j+\sum_{j\in \operatorname{obs}(\masksample'_{\istar})}L_j .
\end{aligned}
\end{align*}
Now, let us sort the $L_j$ for $j \in [d]$ by non-increasing order
  \begin{align*}
    L_{(1)} \geq L_{(2)} \geq \dots \geq L_{(d)}.
  \end{align*}
By simply picking the $\lfloor \rho d \rfloor$ largest FWL constants, we obtain
\begin{align*}
    \sum_{j\in \operatorname{obs}(\masksample_{\istar})}L_j
    \le
    \sum_{j=1}^{\lfloor\rho d\rfloor}L_{(j)},
    \qquad
    \sum_{j\in \operatorname{obs}(\masksample'_{\istar})}L_j
    \le
    \sum_{j=1}^{\lfloor\rho d\rfloor}L_{(j)}.
\end{align*}
Combining the previous inequalities gives
\begin{align*}
\left\|
    \query(\tilde{\data}(\mask))
    -
    \query(\tilde{\data}'(\mask'))
\right\|_p
\le
2\sum_{j=1}^{\lfloor\rho d\rfloor}L_{(j)}.
\end{align*}
Taking the supremum over all admissible neighboring datasets and masks
concludes the proof.
\end{proof}
  
\subsection{Proof of Theorem~\ref{thm:fwllaplace}}\label{app:lap}
\vspace{3ex}
\begin{mdframed}
  \TheoremFWLLap*
\end{mdframed}

\begin{tcolorbox}[boxrule=0pt,standard jigsaw, opacityback=0, frame hidden,sharp corners,enhanced,breakable,borderline west={1pt}{0pt}{gray},left=6pt,right=0pt,top=0pt,bottom=0pt,boxsep=1pt]
\begin{proof}[Sketch of proof] We provide below a sketch of proof for the above result. The full proof can be found in Appendix~\ref{app:lap}.

\textbf{I)} Fix an arbitrary mask $\mask \in \{0,1\}^{n\times d}$ and consider the query applied to incomplete datasets under this fixed mask.
Since the query $\query$ is feature-wise Lipschitz (FWL) with respect to $\|\cdot\|_1$, \Cref{lemma:sensmiss} implies that its
$\ell_1$-sensitivity under the mask $\mask$ is upper-bounded as $
\tilde{\Delta}^{(1)}_f \le \tilde{C}_1 \le C_1.
$
Recall that the base Laplace mechanism $\lap$ with scale parameter $b = C_1/\epsilon$ is $(\epsilon,0)$-DP on complete datasets.
By \Cref{lemma:la} in the Appendix, when applied to incomplete datasets sharing the same mask $\mask$, this Laplace mechanism is also 
$\big((\tilde{C}_1/C_1)\epsilon,0\big)$-DP. for any neighboring pair
$\tilde{\data}(\mask) \simeq \tilde{\data}'(\mask)$ in $\missdatadomain^n$.

 \textbf{II)} 
 We now relax the fixed-mask condition and consider the full missing-data process$\lapmiss$.
 Following \Cref{app:proof}, the distribution of $\lapmiss$ is a mixture distribution over all possible masks from the missing data distribution $\datamiss$.
 Applying the same mixture-based differential privacy analysis as in the proof of \Cref{thm:amplif} in \Cref{app:proof}, we  can obtain the desired result as stated in \Cref{thm:fwllaplace}.
\end{proof}
\end{tcolorbox}

\begin{proof}

To prove the theorem, we first establish a privacy guarantee for the Laplace mechanism under missing data $\lapmiss$, conditional on a fixed mask. We then bound the mixture distribution of $\lapmiss$ to conclude.

\textbf{Privacy guarantee for $\lapmiss$ with a fixed mask $\mask, \mask'$.} 

Let us fix a mask $\mask\in \textrm{supp}(\datamiss(\data))$ and $\mask' \in \textrm{supp}(\datamiss(\data'))$ such that $\tilde{\data}(\mask) \simeq \tilde{\data'}(\mask')$. Since the query $\query$ is FWL, from \Cref{lemma:sensmiss}, we have 
\begin{align*}
  \querysensmiss{1}\leq \tilde{C}_1  = C \sum_{j=1}^{\lfloor\rho d \rfloor} L_{(j)}.
\end{align*}
Since the Laplace mechanism is calibrated with scale $b = \frac{C_1}{\epsilon}$ with $C_1 = 2\sum_{j=1}^{d} L_j$
and since the masked sensitivity is bounded by
\begin{align*}
  \widetilde C_1 = 2\sum_{j=1}^{\lfloor \rho d \rfloor} L_{(j)},
\end{align*}
Lemma~\ref{lemma:la} gives, for any measurable $S \in \mathcal{E}$ and any neighboring datasets
$\data \simeq \data'$ in $\datadomain^n$,
\begin{align*}
  \Prob{\lap(\tilde{\data}(\mask)) \in S}
  \leq
  e^{\tilde{\epsilon}}
  \Prob{\lap(\tilde{\data}'(\mask')) \in S},
\end{align*}
where $\tilde{\epsilon} = (\widetilde C_1 /{C_1})\, \epsilon$.

\textbf{Bounding the mixture distribution of $\Miss$.}

We have for any dataset $\data$ in $\datadomain$ and any measurable $S$ in $\mathcal{E}$
\begin{align*}
  \Prob{\lapmiss(\data) \in S} &= \sum_{\mask \in \{0,1\}^{n\times d}}
    \Prob{\datamiss(\data) = \mask}\Prob{\lap\left(\tilde{\data}(\mask) \right) \in S}
\end{align*}

Similarly to the proof of Theorem~\ref{thm:amplif} (see Appendix~\ref{app:proof}), we perform the following decomposition for all $S \in \mathcal{E}$
\begin{align*}
  &\Prob{\lapmiss(\data) \in S} = (1-p_{\ast}).w^{(0)}(S) + \qstar.w^{(1)}_{\data}(S) ,\\
  &\Prob{\lapmiss(\data') \in S} = (1-p_{\ast}).w^{(0)}(S) + \qstar.w^{(1)}_{\data'}(S).
\end{align*}

Applying Proposition~\ref{prop:conv}, \Cref{prop:advancedconv} and the convexity properties of the $\alpha$-divergence gives
\begin{align*}
  \mathcal{D}_{e^{\epsilon'}}(\lapmiss(\data) || \lapmiss(\data')) \leq \qstar \big[(1-\beta) \mathcal{D}_{e^{\tilde{\epsilon}}}(w^{(1)}_{\data} || w^{(0)}) + \beta \mathcal{D}_{e^{\tilde{\epsilon}}}(w^{(1)}_{\data} || w^{(1)}_{\data'}) \big],
\end{align*}

with $\beta = e^{\epsilon' - \tilde{\epsilon}}$.

\textbf{Bounding $\mathcal{D}_{e^{{\tilde{\epsilon}}}}(w^{(1)}_{\data} || w^{(0)})$.}

The coupling argument from Appendix~\ref{app:proof} applies verbatim with $\delta=0$. Using the same coupling $\pi_1 \in \mathcal{C}(\mu^{(0)}_{\data},\mu^{(1)}_{\data})$, we obtain
\begin{align*}
  \mathcal{D}_{e^{\tilde{\epsilon}}}(w^{(1)}_{\data} || w^{(0)}) \leq 0.
\end{align*}

\textbf{Bounding $\mathcal{D}_{e^{\tilde{\epsilon}}}(w^{(1)}_{\data} || w^{(1)}_{\data'})$.}

Again, the coupling argument from Appendix~\ref{app:proof} applies verbatim with $\delta=0$. Using the same coupling $\pi_2 \in \mathcal{C}(\mu^{(1)}_{\data},\mu^{(1)}_{\data'})$, we obtain
\begin{align*}
  \mathcal{D}_{e^{\tilde{\epsilon}}}(w^{(1)}_{\data} || w^{(1)}_{\data'}) \leq 0.
\end{align*}

By combining the two bounds we obtain
\begin{align*}
  \mathcal{D}_{e^{\epsilon'}}(\lapmiss(\data) || \lapmiss(\data')) \leq  0,
\end{align*}
and conclude by taking the supremum over all pairs neighboring datasets $\data \simeq \data'$ in $\datadomain^n$.


\end{proof}

\subsection{Corollary of \Cref{thm:fwllaplace}}\label{corr:fwllaplace}
\vspace{3ex}
\begin{mdframed}
\begin{restatable}{corollary}{ErwanCor}
\label{cor:laplace}
Under the setting and assumptions of \Cref{thm:fwllaplace}, let $\epsilon \in (0,1]$. Grant  \Cref{ass:proportion_observed_entries}, and assume that  all $L_j$ for the $\ell_1$ norm are equal and $\rho d \in \mathds{N}$. Then, the randomized algorithm $\lapmiss$ is $(\epsilon_0',0)$-DP, with 
\begin{align*}
    \epsilon_0' \leq  \min\left((e-1) p_{\ast}, 1\right) \rho \epsilon.
\end{align*}
\end{restatable}
\end{mdframed}
\begin{proof}
The bound follows from concavity of $\log$ (yielding $\log(1+x)\le x$ for $x\ge 0$) and convexity of $\exp$ (yielding $e^{x}-1\le (e-1)x$ for $0\le x\le 1$), plus the inequality $1+\qstar(e^{x}-1)\le e^{x}$.
\end{proof}

\subsection{Proof of Theorem~\ref{thm:gaussamplif}}\label{app:gaussamplif}

\vspace{3ex}
\begin{mdframed}
  \GaussianAmplif*
\end{mdframed}

\begin{proof}
  The proof is identical to Theorem~\ref{thm:fwllaplace} (see Appendix~\ref{app:lap}).

  \textbf{Privacy guarantee for $\guassmiss$ with a fixed mask $\mask, \mask'$.} 

  For fixed mask $\mask,\mask' \in \textrm{supp}(\datamiss(\data), \textrm{supp}(\datamiss(\data')$, similarly to Appendix~\ref{app:lap} by applying Lemma~\ref{lemma:gauss}, we obtain for any measurable $S$ in $\mathcal{E}$ and any dataset $\data$ in $\datadomain^n$
\begin{align}
  \Prob{\gauss(\tilde{\data}(\mask)) \in S} \leq e^{\tilde{\epsilon}} \Prob{\gauss(\tilde{\data}'(\mask')) \in S} + \delta,
\end{align}

where $0 <\tilde{\epsilon}= \frac{\tilde{C}_2}{C_2}\epsilon \leq 1$ and $\delta \in (0,1]$. 

\textbf{Bounding the mixture distribution of $\guassmiss$.}

We reuse the same exact arguments from Appendix~\ref{app:lap} with $\delta \in (0,1]$. 

We first obtain
\begin{align}
  \mathcal{D}_{e^{\epsilon'}}(\guassmiss(\data) || \guassmiss(\data')) \leq \qstar \big[(1-\beta) \mathcal{D}_{e^{\tilde{\epsilon}}}(w_1 || w_0) + \beta \mathcal{D}_{e^{\tilde{\epsilon}}}(w_1 || w'_1) \big],
\end{align}
Then, we show that
\begin{align}
\mathcal{D}_{e^{\tilde{\epsilon}}}(w_1 || w_0) \leq \delta,
\end{align}
and
\begin{align}
\mathcal{D}_{e^{\tilde{\epsilon}}}(w_1 || w'_1) \leq \delta.
\end{align}

Finally, we conclude by combining the latter bounds and taking the supremum over any neighboring datasets $\data, \data' \in \datadomain^n$ .

\end{proof}

\subsection{Corollary of \Cref{thm:gaussamplif}}\
\vspace{3ex}
\begin{mdframed}
\begin{restatable}{corollary}{ErwanCor}
\label{cor:laplace}
Under the setting and assumptions of \Cref{thm:gaussamplif}, let $\epsilon \in (0,1]$. Grant  \Cref{ass:proportion_observed_entries}, and assume that  all $L_j$ for the $\ell_2$ norm are equal and $\rho d \in \mathds{N}$. Then, the randomized algorithm $\guassmiss$ is $(\epsilon_0',\delta')$-DP, with 
\begin{align*}
    \epsilon_0' \leq  \min\left((e-1) p_{\ast}, 1\right) \rho \epsilon, \quad \delta' = \qstar \delta. 
\end{align*}
\end{restatable}
\end{mdframed}
\begin{proof}
The bound follows from concavity of $\log$ (yielding $\log(1+x)\le x$ for $x\ge 0$) and convexity of $\exp$ (yielding $e^{x}-1\le (e-1)x$ for $0\le x\le 1$), plus the inequality $1+\qstar(e^{x}-1)\le e^{x}$.
\end{proof}

\section{Amplification in real-world dataset}\label{app:rwdata}
\begin{table}[!ht]
\centering
\caption{
Empirical missingness configurations in standard incomplete-data benchmarks.
Dataset sources: \texttt{nhanes} and \texttt{selfreport} from \texttt{mice}~\citep{vanbuuren2011mice};\texttt{housevotes84} from \texttt{mlbench}~\citep{leisch2026mlbench}; \texttt{vnf} from \texttt{missMDA} ~\citep{josse2016missmda}; \texttt{MovieLens20M} ~\citep{movielens20m}.
}
\label{tab:real-world-missingness}
\small
\begin{tabular}{llcc}
\toprule
Dataset & Missingness origin & $\qstar$ & $\rho$ \\
\midrule
\texttt{nhanes}       & Health survey skip patterns        & 0.72 & 1.00 \\
\texttt{housevotes84} & Voting-record abstentions          & 0.92 & 1.00 \\
\texttt{vnf}          & Questionnaire skipped questions    & 0.95 & 1.00 \\
\texttt{selfreport}   & Sensitive-survey skipped items     & 1.00 & 0.79 \\
\texttt{MovieLens20M} & Unrated user--movie pairs  & 1.00 & 0.33 \\
\bottomrule
\end{tabular}
\end{table}

In Table~\ref{tab:real-world-missingness} we shed light that the structural quantities appearing in our amplification results, namely $\qstar$ and $\rho$, arise naturally in standard missing-data settings. The table displays two practically relevant regimes. First, in datasets such as \texttt{nhanes}, \texttt{housevotes84}, and \texttt{vnf}, the empirical value of $\qstar$ is strictly below one, meaning that fully missing rows occur with nonzero frequency. This is precisely the regime in which the generic amplification result from \Cref{thm:amplif} can be non-vacuous, since the differing row may be completely hidden by the missingness process. Second, in \texttt{selfreport} and \texttt{MovieLens20M}, all rows are at least partially observed, so $\qstar = 1$, but the maximal observed-feature fraction satisfies $\rho<1$. This corresponds to the regime targeted by the amplification result from \Cref{thm:fwllaplace} and \Cref{thm:gaussamplif}: even when the differing row is never fully masked, its contribution to query sensitivity is reduced because no row is fully observed. Overall, these configurations support the practical relevance of our assumptions: missingness patterns encountered in health surveys, questionnaires, voting records, longitudinal studies, and diagnostic datasets can naturally activate either the generic fully-masked-row effect or the feature-wise sensitivity reduction captured by our theorems.

\section{Relaxed \Cref{ass:proportion_observed_entries} MCAR results}\label{app:relaxed-fwl}

To relax the requirement of \Cref{ass:proportion_observed_entries} in both \Cref{thm:fwllaplace} and \Cref{thm:gaussamplif}, we allow the bounded-observation condition
to fail on a small fraction of masks. More precisely, we isolate the masks
for which the number of observed features in the differing row exceeds
the threshold $\rho d$. For any neighboring datasets $\data \simeq \data'$ with $\istar$ being their differing row and for any $\rho \in [0,1]$, we define
\begin{align*}
    &B_{\rho}(\data, \data') =   \big\{\mask \in \Hstar(\data, \data') \mid \exists i \in [n],\textrm{Card}(\operatorname{obs}\left(\masksample_i)\right) > \rho d\big\}.
\end{align*}

Formally, let us consider the following assumption. 

\begin{assumption}
  \label{ass:relaxed_rho} The missing data process $\datamiss = \featuremiss^{\otimes n}$ is such that there exists $\rho \in [0, 1]$ and $\eta \in [0,1)$ such that for any pair of neighboring datasets $\data \simeq  \data' \in \datadomain^n$, 
  \begin{align*}
\Prob{\datamiss(\data)\in B_\rho(\data,\data')}
\le \eta.
\end{align*}
\end{assumption}

The relaxed FWL results presented in this section should be understood as a robustness extension of the deterministic bounded-observation condition. When this condition holds uniformly, the main FWL theorem already yields amplification with no exceptional term. The relaxed analysis instead addresses datasets where the condition fails on a subset of masks, and quantifies exactly how these violations enter the privacy profile through the tail mass of the row-wise observed-fraction distribution. When this upper tail is light, the main amplification effect approximately persists up to a small exceptional mass; when it is heavy, the bound correctly indicates that missingness provides limited practical amplification.

For the Laplace mechanism, we can show the following result with relaxed assumption.

\vspace{3ex}
\begin{mdframed}
\begin{restatable}{theorem}{LapRelaxedMCAR}\label{thm:relaxedfwllap}
Let $\epsilon >0$. Let $\datamiss$ be a MCAR missing data process verifying \Cref{ass:relaxed_rho}. Let $\query$ be a FWL query with respect to $\|.\|_1$, and $\lap$ be the Laplace mechanism, associated to $\query$ with scale parameter $b=C_1 / \epsilon$. 
Then the randomized algorithm $\lapmiss$ defined as per~(\ref{def:A_D}) with $\lap$ and $\datamiss$ is $(\epsilon_\rho^{\star}, \eta)$-DP with 
\begin{align*}
\epsilon^\star_\rho
=
\ln\!\Bigl(1+\frac{p_\ast-r_\rho}{1-r_\rho}(e^{\tilde{\epsilon}}-1)\Bigr),
\end{align*}
where $\tilde{\epsilon} = (\tilde C_1/C_1)\epsilon$ and $r_{\rho} = \Prob{\datamiss(\data) \in B_{\rho}(\data, \data')}=\Prob{\datamiss(\data') \in B_{\rho}(\data, \data')}$ for any pair of neighboring datasets $\data \simeq \data' \in \datadomain^n$.
\end{restatable}
\end{mdframed}
\begin{proof}

Let us introduce the following notation $G_\rho(\data,\data') = \Hstar(\data,\data')\setminus B_\rho(\data,\data')$ for conciseness.
Since by definition, 
\begin{align*}
B_\rho(\data,\data')\subseteq \Hstar(\data,\data'),
\end{align*}
we obtain the subsequent disjoint decomposition
\begin{align}
\{0,1\}^{n\times d}
=
\Hstar(\data,\data')^c
\cup
G_\rho(\data,\data')
\cup
B_\rho(\data,\data').\label{eq:setdecomp}
\end{align}

In addition, since $\datamiss$ is MCAR, the mask law does not depend on the data, hence there exists $r_\rho$ such that for all pairs of neighboring datasets $\data \simeq \data'$ and all $\mask \in \{0,1\}^{n \times d}$, we have
\begin{align}
r_\rho = \Prob{\datamiss(\data) \in B_{\rho}(\data, \data')}=\Prob{\datamiss(\data') \in B_{\rho}(\data, \data')}. \label{eq:defrrho}
\end{align}

The mass of $G_{\rho}(\data, \data')$ for all pairs of neighboring datasets $\data \simeq \data'$ is given by
\begin{align*}
    \Prob{\datamiss(\data)\in G_\rho(\data,\data')}
=
\Prob{\datamiss(\data')\in G_\rho(\data,\data')}
=
p_\ast-r_\rho.
\end{align*}

Let us fix a pair of neighboring dataset $\data \simeq \data'$ as well as an arbitrary event $S \in \mathcal{E}$.

\textbf{I) Privacy guarantee in $G_{\rho}(\data, \data')$.}

Let $\mask\in G_\rho(\data,\data')$. By definition, we have for all $i \in [n]$
\begin{align}
\textrm{Card}(\operatorname{obs}(\mask_{i}))\le \rho d.
\end{align}
Therefore
\begin{align}
\textrm{Card}(\operatorname{obs}(\mask_{\istar}))\le \rho d.
\end{align}
Hence \cref{ass:relaxed_rho} holds for the fixed mask $\mask$.
Since $\query$ is FWL with respect to $\|\cdot\|_1$, 
applying Lemma~\ref{lemma:sensmiss} yields
\begin{align*}
\|\query(\tilde{\data}(\mask))-\query(\tilde{\data}'(\mask))\|_1
\le
2B\sum_{j=1}^{\lfloor \rho d\rfloor}L_{(j)}
=
\tilde C_1.
\end{align*}
Therefore, by \Cref{lemma:la} we obtain
\begin{align}
\Prob{\lap(\tilde{\data}(\mask))\in S}
\le
e^{\tilde{\epsilon}}\Prob{\lap(\tilde{\data}'(\mask))\in S},
\qquad \text{with} \quad
\tilde{\epsilon}=(\tilde C_1/C_1)\epsilon.\label{ineq:grho}
\end{align}

We now justify the privacy bound on the conditional laws over
$G_\rho(\data,\data')$.

By definition of conditional probability,
\begin{align*}
\Prob{\lapmiss(\data)\in S \mid \datamiss(\data)\in G_\rho(\data, \data')}=
\sum_{\mask\in G_\rho(\data, \data')}
\Prob{\datamiss(\data)=\mask}
\,\Prob{\lap(\tilde{\data}(\mask))\in S}.
\end{align*}
Similarly,
\begin{align*}
\Prob{\lapmiss(\data')\in S \mid \datamiss(\data')\in G_\rho(\data, \data')}
&=
\sum_{\mask\in G_\rho(\data, \data')}
\Prob{\datamiss(\data')=\mask }
\,\Prob{\lap(\tilde{\data}'(\mask))\in S}.
\end{align*}

Since $\datamiss$ is MCAR, for every $\mask\in G_\rho(\data, \data')$ we have
\begin{align*}
\Prob{\datamiss(\data)=\mask}=\Prob{\datamiss(\data')=\mask}.
\end{align*}

Combining the latter inequality with \Cref{ineq:grho} and summing over
$\mask\in G_\rho(\data,\data')$, we obtain
\begin{align*}
\Prob{\lapmiss(\data)\in S \mid \datamiss(\data)\in G_\rho(\data, \data')}
&\le
e^{\epsilon_0}
\sum_{\mask\in G_\rho(\data, \data')}
\Prob{\datamiss(\data)=\mask }
\,\Prob{\lap(\tilde{\data}'(\mask))\in S}
\\
& = e^{\epsilon_0}
\sum_{\mask\in G_\rho(\data, \data')}
\Prob{\datamiss(\data')=\mask }
\,\Prob{\lap(\tilde{\data}'(\mask))\in S}.
\end{align*}
Hence,
\begin{align}\label{eq:grhodp}
\Prob{\lapmiss(\data)\in S \mid \datamiss(\data)\in G_\rho(\data, \data')}\le e^{\tilde{\epsilon}}\Prob{\lapmiss(\data')\in S \mid \datamiss(\data')\in G_\rho(\data, \data')}.
\end{align}

\textbf{II) $\lapmiss$ conditional distribution by $B_{\rho}^c(\data, \data')$}

We first have
\begin{align*}
B_\rho^c(\data,\data')
=
\Hstar(\data,\data')^c\cup G_\rho(\data,\data').
\end{align*}
By \Cref{eq:defrrho}, we also have
\begin{align*}
\Prob{\datamiss(\data)\in B_\rho^c(\data,\data')}
=
\Prob{\datamiss(\data')\in B_\rho^c(\data,\data')}
=
1-r_\rho.
\end{align*}

We now expand
\begin{align*}
\Prob{\lapmiss(\data)\in S \mid \datamiss(\data)\in B_\rho^c(\data,\data')}
\end{align*}
using the conditional probability formula and the law of total probability on the partition
\begin{align*}
B_\rho^c(\data,\data')
=
\Hstar^c(\data,\data')
\cup
G_\rho(\data,\data').
\end{align*}

Since
\begin{align*}
B_\rho^c(\data,\data')
=
\Hstar^c(\data,\data')
\cup
G_\rho(\data,\data') \quad \text{and} \quad
\Hstar^c(\data,\data')\cap G_\rho(\data,\data')=\varnothing,
\end{align*}
the event$
\{\datamiss(\data)\in B_\rho^c(\data,\data')\}$
is partitioned into the two disjoint events
$
\{\datamiss(\data)\in \Hstar^c(\data,\data')\}$ and $\{\datamiss(\data)\in G_\rho(\data,\data')\}$. Therefore, applying the law of total probability inside the conditioning event
$\{\datamiss(\data)\in B_\rho^c(\data,\data')\}$, we obtain
\begin{align*}
&\Prob{\lapmiss(\data)\in S \mid \datamiss(\data)\in B_\rho^c(\data,\data')}
\\
&=
\frac{
\Prob{\{\lapmiss(\data)\in S\}\cap\{\datamiss(\data)\in \Hstar^c(\data,\data')\}}
+
\Prob{\{\lapmiss(\data)\in S\}\cap\{\datamiss(\data)\in G_\rho(\data,\data')\}}
}{
\Prob{\datamiss(\data)\in B_\rho^c(\data,\data')}
} \\&= \frac{
\Prob{\{\lapmiss(\data)\in S\}\cap\{\datamiss(\data)\in \Hstar^c(\data,\data')\}}
+
\Prob{\{\lapmiss(\data)\in S\}\cap\{\datamiss(\data)\in G_\rho(\data,\data')\}}
}{1- r_\rho}
\end{align*}
Using
\begin{align*}
\mathbb{P}(A\cap C)=\mathbb{P}(A\mid C)\mathbb{P}(C),
\end{align*}
First, we have
\begin{align*}
\Prob{\{\lapmiss(\data)\in S\}\cap\{\datamiss(\data)\in \Hstar^c(\data,\data')\}} &= \Prob{\lapmiss(\data)\in S \mid \datamiss(\data)\in \Hstar^c(\data,\data')}
\,\Prob{\datamiss(\data)\in \Hstar^c(\data,\data')} \\&= \Prob{\lapmiss(\data)\in S \mid \datamiss(\data)\in \Hstar^c(\data,\data')} (1-\qstar)
\end{align*}

and
\begin{align*}
    \Prob{\{\lapmiss(\data)\in S\}\cap\{\datamiss(\data)\in G_{\rho}(\data,\data')\}} &= \Prob{\lapmiss(\data)\in S \mid \datamiss(\data)\in G_\rho(\data,\data')}
\,\Prob{\datamiss(\data)\in G_\rho(\data,\data')}
\\& = \Prob{\lapmiss(\data)\in S \mid \datamiss(\data)\in G_\rho(\data,\data')}(\qstar - r_{\rho}).
\end{align*}

Thus substituting gives
\begin{align}
\Prob{\lapmiss(\data)\in S \mid \datamiss(\data)\in B_\rho^c(\data,\data')}
&=
\frac{1-p_\ast}{1-r_\rho}
\Prob{\lapmiss(\data)\in S \mid \datamiss(\data)\in \Hstar^c(\data,\data')}
\nonumber\\
&\quad+
\frac{p_\ast-r_\rho}{1-r_\rho}
\Prob{\lapmiss(\data)\in S \mid \datamiss(\data)\in G_\rho(\data,\data')},
\label{eq:condmix-good-z}
\end{align}

Similarly, we have
\begin{align}
\Prob{\lapmiss(\data')\in S \mid \datamiss(\data')\in B_\rho^c(\data,\data')}
&=
\frac{1-p_\ast}{1-r_\rho}
\Prob{\lapmiss(\data')\in S \mid \datamiss(\data')\in \Hstar^c(\data,\data')}
\nonumber\\
&\quad+
\frac{p_\ast-r_\rho}{1-r_\rho}
\Prob{\lapmiss(\data')\in S \mid \datamiss(\data')\in G_\rho(\data,\data')}.
\label{eq:condmix-good-zprime}
\end{align}

Then we have by definition of conditional probability
\begin{align*}
\Prob{\lapmiss(\data)\in S \mid \datamiss(\data)\in \Hstar^c(\data,\data')} &= \frac{
\sum_{\mask\in \Hstar^c(\data,\data')}
\Prob{\datamiss(\data)=\mask}
\Prob{\lap(\tilde{\data}(\mask))\in S}
}{
\Prob{\datamiss(\data)\in \Hstar^c(\data,\data')}
}\ \\
&= \frac{
\sum_{\mask\in \Hstar^c(\data,\data')}
\Prob{\datamiss(\data)=\mask}
\Prob{\lap(\tilde{\data}(\mask))\in S}
}{1-\qstar}.
\end{align*}
If $\mask\in \Hstar^c(\data,\data')$, then the differing row is fully masked, so $\tilde{\data}(\mask)=\tilde{\data}'(\mask)$.

Besides, since $\datamiss$ is MCAR, we have
\begin{align*}
\forall \mask\in \Hstar^c(\data,\data'), \quad \Prob{\datamiss(\data)=\mask}=\Prob{\datamiss(\data')=\mask}.
\end{align*}

Therefore,
\begin{align*}
    \Prob{\lapmiss(\data)\in S \mid \datamiss(\data)\in \Hstar^c(\data,\data')} = \Prob{\lapmiss(\data')\in S \mid \datamiss(\data')\in \Hstar^c(\data,\data')}.
\end{align*}

For convenience, we now introduce
\begin{align*}
w_0(S)
&=
\Prob{\lapmiss(\data)\in S \mid \datamiss(\data)\in \Hstar^c(\data,\data')} = \Prob{\lapmiss(\data')\in S \mid \datamiss(\data')\in \Hstar^c(\data,\data')}, 
\\
w_1(S)
&=
\Prob{\lapmiss(\data)\in S \mid \datamiss(\data)\in G_\rho(\data,\data')},
\\
w_1'(S)
&=
\Prob{\lapmiss(\data')\in S \mid \datamiss(\data')\in G_\rho(\data,\data')}.
\end{align*}

The previous equalities \eqref{eq:condmix-good-z} and \eqref{eq:condmix-good-zprime} rewrite as
\begin{align*}
\Prob{\lapmiss(\data)\in S \mid \datamiss(\data)\in B_\rho^c(\data,\data')}
&=
\frac{1-p_\ast}{1-r_\rho}w_0(S)
+
\frac{p_\ast-r_\rho}{1-r_\rho}w_1(S),
\\
\Prob{\lapmiss(\data')\in S \mid \datamiss(\data')\in B_\rho^c(\data,\data')}
&=
\frac{1-p_\ast}{1-r_\rho}w_0(S)
+
\frac{p_\ast-r_\rho}{1-r_\rho}w_1'(S).
\end{align*}

These two conditional mechanisms are therefore in exactly the same form as in
the proof of Theorem~\ref{thm:fwllaplace}, with $\frac{p_\ast-r_\rho}{1-r_\rho}$ instead of $\qstar$.
Applying \Cref{eq:grhodp}, Proposition~\ref{prop:conv} and the same mixture argument as in the proof of
Theorem~\ref{thm:fwllaplace}, with
\begin{align*}
\epsilon^\star_\rho
=
\ln\!\Bigl(1+\frac{p_\ast-r_\rho}{1-r_\rho}(e^{\tilde{\epsilon}}-1)\Bigr),
\end{align*}
yields
\begin{align}
\Prob{\lapmiss(\data)\in S\mid \datamiss(\data)\in B_\rho^c(\data,\data')}
\le
e^{\epsilon^\star_\rho}
\Prob{\lapmiss(\data')\in S\mid \datamiss(\data')\in B_\rho^c(\data,\data')}, \label{eq:conditional-good-mcar}
\end{align}
where
\begin{align*}
\epsilon^\star_\rho
=
\ln\!\Bigl(1+\frac{p_\ast-r_\rho}{1-r_\rho}(e^{\tilde{\epsilon}}-1)\Bigr).
\end{align*}

\textbf{III) Full mixture decomposition.}

By conditioning on $B_\rho^c(\data,\data')$ and $B_\rho(\data,\data')$, we have
\begin{align*}
\Prob{\lapmiss(\data)\in S}
&=
(1-r_\rho)\Prob{\lapmiss(\data)\in S\mid \datamiss(\data)\in B_\rho^c(\data,\data')}
\\
&\quad+
r_\rho\Prob{\lapmiss(\data)\in S\mid \datamiss(\data)\in B_\rho(\data,\data')}.
\end{align*}
Since conditional probabilities are bounded by $1$, we obtain
\begin{align*}
\Prob{\lapmiss(\data)\in S}
\le
(1-r_\rho)\Prob{\lapmiss(\data)\in S\mid \datamiss(\data)\in B_\rho^c(\data,\data')}
+r_\rho.
\end{align*}
Using \Cref{eq:conditional-good-mcar}, this gives
\begin{align*}
\Prob{\lapmiss(\data)\in S}
&\le
(1-r_\rho)e^{\epsilon^\star_\rho}
\Prob{\lapmiss(\data')\in S\mid \datamiss(\data')\in B_\rho^c(\data,\data')}
+r_\rho.
\end{align*}
Besides,
\begin{align*}
(1-r_\rho)\Prob{\lapmiss(\data')\in S\mid \datamiss(\data')\in B_\rho^c(\data,\data')}
\le
\Prob{\lapmiss(\data')\in S},
\end{align*}
so
\begin{align*}
\Prob{\lapmiss(\data)\in S}
\le
e^{\epsilon^\star_\rho}\Prob{\lapmiss(\data')\in S}+r_\rho.
\end{align*}

Since $S$ was arbitrary,
\begin{align*}
D_{e^{\epsilon_\rho^\star}}(\lapmiss(\data)\|\lapmiss(\data'))\le r_\rho.
\end{align*}
By Assumption~\ref{ass:relaxed_rho},
\begin{align*}
r_\rho
=
\Prob{\datamiss(\data)\in B_\rho(\data,\data')}
\le \eta.
\end{align*}
Therefore,
\begin{align*}
D_{e^{\epsilon^\star_\rho}}(\lapmiss(\data)\|\lapmiss(\data'))\le \eta.
\end{align*}
Taking the supremum over neighboring datasets $\data \simeq \data'$ yields
\begin{align*}
\sup_{\data \simeq \data'}
D_{e^{\epsilon^\star_\rho}}(\lapmiss(\data)\|\lapmiss(\data'))
\le \eta,
\end{align*}
which shows that $\lapmiss$ is $(\epsilon_\rho^\star,\eta)$-differentially private.

\end{proof}

Analogously, for the Gaussian mechanism, we establish the following amplification result under a relaxed assumption.

\vspace{3ex}
\begin{mdframed}
\begin{restatable}{theorem}{GaussRelaxedMCAR}
Let $\datadomain \subset [-1,1]^d$, and let $\epsilon \in (0,1]$ and $\delta \in [0,1]$. Let $\mathcal{D}$ be a MCAR missing data process verifying \Cref{ass:relaxed_rho}. Let $\query$ be a FWL query with respect to $\|.\|_2$, and $\gauss$ be the  Gaussian mechanism, associated to $\query$ with scale parameter $\sigma = c \, C_2 / \epsilon$. 
  Then the randomized algorithm $\guassmiss$ defined as per~(\ref{def:A_D}) with $\gauss$ and $\datamiss$ is $(\epsilon_\rho^{\star}, \eta  + \qstar\delta)$-DP with
  \begin{align*}
\epsilon^\star_\rho
=
\ln\!\Bigl(1+\frac{p_\ast-r_\rho}{1-r_\rho}(e^{\tilde{\epsilon}}-1)\Bigr),
\end{align*}
where $\tilde{\epsilon} = (\tilde C_2/C_2)\epsilon$. 
\end{restatable}
\end{mdframed}

\begin{proof}
The proof follows the same decomposition as in the proof of
\Cref{thm:relaxedfwllap}. We therefore only highlight the modification
needed for the Gaussian mechanism.

Let $\data\simeq\data'$ be neighboring datasets, let $\istar$ be their
differing row, and define
\begin{align*}
G_\rho(\data,\data')
=
\Hstar(\data,\data')\setminus B_\rho(\data,\data').
\end{align*}
As in the Laplace case, since $\datamiss$ is MCAR, the law of the mask is
data-independent and we have
\begin{align*}
\Prob{\datamiss(\data)\in B_\rho(\data,\data')}
=
\Prob{\datamiss(\data')\in B_\rho(\data,\data')}
=
r_\rho
\le \eta,
\end{align*}
and
\begin{align*}
\Prob{\datamiss(\data)\in G_\rho(\data,\data')}
=
\Prob{\datamiss(\data')\in G_\rho(\data,\data')}
=
p_\ast-r_\rho.
\end{align*}

For any $\mask\in G_\rho(\data,\data')$, the differing row has at most
$\rho d$ observed entries. Since $\query$ is FWL with respect to
$\|\cdot\|_2$, \Cref{lemma:sensmiss} gives
\begin{align*}
\|\query(\tilde\data(\mask))-\query(\tilde\data'(\mask))\|_2
\le \tilde C_2.
\end{align*}
Thus, by the Gaussian mechanism calibration with
$\sigma=cC_2/\epsilon$, the fixed-mask mechanism satisfies
\begin{align*}
\Prob{\gauss(\tilde\data(\mask))\in S}
\le
e^{\tilde{\epsilon}}
\Prob{\gauss(\tilde\data'(\mask))\in S}
+\delta,
\qquad
\tilde{\epsilon}=(\tilde C_2/C_2)\epsilon .
\end{align*}

Averaging this inequality over $\mask\in G_\rho(\data,\data')$ and using
the MCAR property gives
\begin{align*}
\Prob{\guassmiss(\data)\in S\mid \datamiss(\data)\in G_\rho}
\le
e^{\tilde{\epsilon}}
\Prob{\guassmiss(\data')\in S\mid \datamiss(\data')\in G_\rho}
+\delta .
\end{align*}
On $\Hstar^c(\data,\data')$, the differing row is fully masked, hence the
conditional laws under $\data$ and $\data'$ coincide, exactly as in the
Laplace proof. Therefore, applying the same mixture argument on
\begin{align*}
B_\rho^c(\data,\data')
=
\Hstar^c(\data,\data')\cup G_\rho(\data,\data')
\end{align*}
yields
\begin{align*}
D_{e^{\epsilon_\rho^\star}}
\Big(
\guassmiss(\data)\mid \datamiss(\data)\in B_\rho^c
\;\Big\|\;
\guassmiss(\data')\mid \datamiss(\data')\in B_\rho^c
\Big)
\le
\frac{p_\ast-r_\rho}{1-r_\rho}\delta,
\end{align*}
where
\begin{align*}
\epsilon^\star_\rho
=
\ln\!\Bigl(
1+\frac{p_\ast-r_\rho}{1-r_\rho}
(e^{\tilde{\epsilon}}-1)
\Bigr).
\end{align*}

Finally, decomposing the full law according to
$B_\rho^c(\data,\data')$ and $B_\rho(\data,\data')$, and upper bounding the
contribution of $B_\rho(\data,\data')$ by its probability $r_\rho$, gives
\begin{align*}
D_{e^{\epsilon_\rho^\star}}
\bigl(
\guassmiss(\data)\|\guassmiss(\data')
\bigr)
\le
(1-r_\rho)\frac{p_\ast-r_\rho}{1-r_\rho}\delta
+
r_\rho
=
(p_\ast-r_\rho)\delta+r_\rho.
\end{align*}
Since $r_\rho\le \eta$, we also have
\begin{align*}
D_{e^{\epsilon_\rho^\star}}
\bigl(
\guassmiss(\data)\|\guassmiss(\data')
\bigr)
\le
p_\ast\delta+\eta.
\end{align*}
Taking the supremum over all neighboring datasets concludes the proof.
\end{proof}

\section{Example of privacy degradation under dependent missingness}\label{app:indepmask}

\textbf{Missing data process.}

Let $n=2$, $d=1$, and consider $\datadomain = \{0,1\}$. Define a missing data process $\datamiss$ that is \emph{not independent across samples} as follows. For any dataset $\data=(\sample_1,\sample_2)\in\datadomain^2$,
\begin{align*}
\datamiss
=
\begin{cases}
(1,1) & \text{if } \sample_1 = 0,\\
(0,0) & \text{if } \sample_1 = 1,
\end{cases}
\end{align*}
where $1$ denotes a missing entry and $0$ an observed entry.
In particular, the mask of the second sample depends on the value of the first sample, introducing dependence across records.

\medskip

\textbf{Query and Laplace mechanism.}

Define the query $\query:\mathcal Z^n_{\mathrm{miss}}\to\mathbb R$ by
\begin{align*}
\query(\tilde\data)
=
\sum_{i=1}^n (1 - \Ind{\tilde{\data}_i = \texttt{NA}}),
\end{align*}
i.e., the number of observed entries.
For any two neighboring incomplete datasets $\tilde\data\simeq\tilde\data' \in \missdatadomain^n$, we have
\begin{align*}
|\query(\tilde\data)-\query(\tilde\data')|
\le
1,
\end{align*}
so the $\ell_1$-sensitivity satisfies $\Delta_1(\query)\le 1$.
Therefore, the Laplace mechanism
\begin{align*}
\algo(\tilde\data)
=
\query(\tilde\data)
+
\mathrm{Lap}(1/\epsilon)
\end{align*}
is $(\epsilon,0)$-differentially private.

\medskip

\textbf{Degradation due to dependent missingness.}
Let $\algo_{\datamiss}$ denote the composed mechanism. Consider the following pair of neighboring datasets
\begin{align*}
\data=(0,0), \qquad \data'=(1,0).
\end{align*}
By construction,
\begin{align*}
\mathcal D(\data)=(1,1)
\quad \text{and} \quad
\mathcal D(\data')=(0,0).
\end{align*}
Therefore
\begin{align*}
\query(\tilde\data)=0,
\qquad
\query(\tilde\data')=2.
\end{align*}

Hence,
\begin{align*}
\algo_{\datamiss}(\data)
&\sim
\mathrm{Lap}(1/\epsilon),
\\
\algo_{\datamiss}(\data')
&\sim
2 + \mathrm{Lap}(1/\epsilon).
\end{align*}

The corresponding densities are
\begin{align*}
g_{\algo_{\datamiss}(\data)}(t)
&=
\frac{\epsilon}{2}e^{-\epsilon |t|},
\qquad
g_{\algo_{\datamiss}(\data')}(t)
=
\frac{\epsilon}{2}e^{-\epsilon |t-2|}.
\end{align*}
By definition,
\begin{align*}
D_{e^\epsilon}
\!\left(
\algo_{\datamiss}(\data)
\;\|\;
\algo_{\datamiss}(\data')
\right)
&=
\sup_{S\in\mathcal B(\mathbb R)}
\left\{
\Prob{\algo_{\datamiss}(\data)\in S}
-
e^\epsilon
\Prob{\algo_{\datamiss}(\data')\in S}
\right\}.
\end{align*}
Equivalently,
\begin{align*}
D_{e^\epsilon}
\!\left(
\algo_{\datamiss}(\data)
\;\|\;
\algo_{\datamiss}(\data')
\right)
&=
\int_{\mathbb R}
\left(
f_{\algo_{\datamiss}(\data)}(t)
-
e^\epsilon f_{\algo_{\datamiss}(\data')}(t)
\right)_+dt .
\end{align*}
We have for any $t \in \mathbb{R}$
\begin{align*}
f_{\algo_{\datamiss}(\data)}(t)
-
e^\epsilon f_{\algo_{\datamiss}(\data')}(t)
&=
\frac{\epsilon}{2}
\left(
e^{-\epsilon |t|}
-
e^\epsilon e^{-\epsilon |t-2|}
\right).
\end{align*}
This quantity is positive if and only if
\begin{align*}
|t-2|-|t|>1\quad \text{i.e.} \quad t < \frac{1}{2}
\end{align*}
Therefore,
\begin{align*}
D_{e^\epsilon}
\!\left(
\algo_{\datamiss}(\data)
\;\|\;
\algo_{\datamiss}(\data')
\right)
&=
\int_{-\infty}^{1/2}
\left(
f_{\algo_{\datamiss}(\data)}(t)
-
e^\epsilon f_{\algo_{\datamiss}(\data')}(t)
\right)dt \\
&= \Prob{\algo_{\datamiss}(\data)\le 1/2}
-
e^\epsilon
\Prob{\algo_{\datamiss}(\data')\le 1/2}.
\end{align*}
Using the Laplace cumulative distribution function,
\begin{align*}
\Prob{\algo_{\datamiss}(\data)\le 1/2}
&=
1-\frac12 e^{-\epsilon/2},
\\
\Prob{\algo_{\datamiss}(\data')\le 1/2}
&=
\frac12 e^{-\epsilon(2-1/2)}
=
\frac12 e^{-3\epsilon/2}.
\end{align*}
Therefore,
\begin{align*}
D_{e^\epsilon}
\!\left(
\algo_{\datamiss}(\data)
\;\|\;
\algo_{\datamiss}(\data')
\right)
&=
1-\frac12 e^{-\epsilon/2}
-
e^\epsilon \frac12 e^{-3\epsilon/2}
\\
&=
1-e^{-\epsilon/2}
>0.
\end{align*}

This shows that, although $\algo$ is $(\epsilon,0)$-DP on neighboring incomplete datasets, the composed mechanism $\algo_{\datamiss}$ does not preserve this guarantee. The dependence in the missingness process creates an additional leakage channel, as the masking pattern of one sample reveals information about another. 

\end{document}